\documentclass[letterpaper]{article} 
\usepackage[preprint]{aaai2027}  
\usepackage[hyphens]{url}  
\usepackage{graphicx} 
\usepackage{natbib}  
\usepackage{caption} 
\usepackage{amsmath,amsfonts,bm}

\def\eqref#1{equation~\ref{#1}}

\def\1{\bm{1}}

\DeclareMathAlphabet{\mathsfit}{\encodingdefault}{\sfdefault}{m}{sl}
\SetMathAlphabet{\mathsfit}{bold}{\encodingdefault}{\sfdefault}{bx}{n}

\newcommand{\E}{\mathbb{E}}

\newcommand{\R}{\mathbb{R}}

\usepackage{amsmath,amssymb,amsthm}
\usepackage{booktabs}
\graphicspath{{figures/}}

\newcommand{\epsdp}{\varepsilon}
\newcommand{\deltadp}{\delta}
\newcommand{\Bo}{B_0}
\newcommand{\lo}{\ell_0}
\newcommand{\N}{\mathcal{N}}
\newcommand{\D}{\mathcal{D}}
\newcommand{\Dpub}{\mathcal{D}_{\mathrm{pub}}}
\newcommand{\clip}{\mathrm{clip}}

\newcommand{\smax}{\sigma_{\max}}

\theoremstyle{plain}
\newtheorem{theorem}{Theorem}[section]
\newtheorem{proposition}[theorem]{Proposition}
\newtheorem{lemma}[theorem]{Lemma}
\newtheorem{corollary}[theorem]{Corollary}
\theoremstyle{definition}

\newtheorem{remark}[theorem]{Remark}
\newtheorem{fact}[theorem]{Fact}

\title{StraightDP: Geometry-Aware Differential Privacy\\for Rectified-Flow Transformers}
\author {
    Xujun Che\textsuperscript{\rm 1},
    Depeng Xu\textsuperscript{\rm 1}\corresponding,
    Xintao Wu\textsuperscript{\rm 2}
}
\affiliations {
    \textsuperscript{\rm 1}Department of Cybersecurity, University of North Carolina at Charlotte, Charlotte, NC, USA\\
    \textsuperscript{\rm 2}Department of Electrical Engineering and Computer Science, University of Arkansas, Fayetteville, AR, USA\\
    xche@charlotte.edu, dxu7@charlotte.edu, xintaowu@uark.edu
}

\begin{document}

\maketitle

\begin{abstract}
Differentially private (DP) training of text-conditioned generative models suffers a utility cliff at strong privacy.
We revisit this problem through the \emph{geometry} of rectified flows: along the straight interpolation between noise and data, the Bayes-optimal velocity is governed to leading order at the noise end by a few class-conditional \emph{moments}, and increasingly sample-specific structure matters toward the data end.
StraightDP exploits this heterogeneity end to end.
A small budget share releases whitened class-conditional moments once, to be distilled into the weights or injected at sampling time. The rest is spent by pre-declared DP-SGD toward the data end, beyond the moments' reach.
At $\epsdp{=}1$ on MNIST, the released moments alone already attain $0.76$ downstream accuracy with prototype-like samples and an FID of $237$, and uniform DP-SGD attains $0.21$.
The pipeline built on the release reaches $0.81$ accuracy at FID $56$ in a public latent space.
Constraining per-token \emph{stream norms} of the multimodal backbone leaves the pretraining loss unchanged yet improves downstream accuracy in the extreme-noise pixel-space regime, and its accuracy effect becomes monotonically more favorable as privacy strengthens.
The released moments also port to frozen SD3-medium, where sampling-time injection beats DP-LoRA training at a fraction of the budget.
\end{abstract}

\section{Introduction}\label{sec:intro}

Text-conditioned generative models are increasingly trained on sensitive paired data: medical images with reports, personal photographs with descriptions.
Such models can memorize, and attacks recover individual training records from weights or samples~\citep{shokri2017_mia,carlini2019_secretsharer}.
Differential privacy (DP)~\citep{dwork2014_book} is the standard remedy: it bounds, in the worst case, how much any single record can influence the released model, and because the guarantee survives arbitrary post-processing, everything a DP model generates can be shared without further privacy loss.
The guarantee comes at the cost of noise, and training generative models under DP is dominated by a single tension: the noise that provides it also erases the signal.
The usual recipe pretrains on public data, by now standard practice for DP learning~\citep{de2022_unlocking,yu2022_dpfinetune,ganesh2023_whypublic}, and fine-tunes with DP-SGD~\citep{abadi2016_dpsgd}, which clips each example's gradient and noises the aggregate. The per-step mechanisms compose under a numerical privacy-loss-distribution (PLD) accountant~\citep{mironov2019_sgm,koskela2020_fft,gopi2021_numerical}.
Across DP generative families (GANs~\citep{jordon2019_pategan,chen2020_gswgan}, one-shot kernel-embedding releases~\citep{harder2021_dpmerf,cao2021_dpsinkhorn,vinaroz2022_dphp,yang2023_dpntk}, and DP diffusion~\citep{dockhorn2022_dpdm,ghalebikesabi2023_dpdiffusion,lyu2023_dpldm}), the private mechanism is largely model-agnostic: each training step is clipped and noised the same way regardless of \emph{what} is being learned at that step.
Adaptive-clipping and dynamic-allocation variants~\citep{andrew2021_adaptiveclipping,bu2022_autoclip,du2021_dynamicdp} adjust the mechanism, but they adapt \emph{during} training on private signal. These pipelines do not explicitly exploit structure knowable \emph{before} any private data is touched (Appendix~\ref{app:related} surveys related work).

Rectified flows~\citep{liu2022_rectifiedflow,lipman2022_flowmatching}, the generative family behind current text-to-image systems~\citep{esser2024_sd3}, learn a velocity field that transports Gaussian noise to data along straight interpolation paths. Sampling integrates the learned field from noise to image.
This family makes the uniformity above appear wasteful.
At each flow time the regression target is a conditional expectation of the noise-to-data displacement given the current point on the straight path.
To leading order near the noise end this field depends only on a few \emph{low-order moments} of the data distribution; toward the data end it depends on increasingly sample-specific structure.
The information the model must extract from private data is therefore \emph{time-heterogeneous}, and so, we argue, should be the privacy budget, the supervision signal, and even the architecture.
StraightDP turns that argument into design decisions on three surfaces that standard DP fine-tuning leaves untouched: \emph{what} is released, \emph{when} the budget is spent, and \emph{on what architecture} (Figure~\ref{fig:overview}).

\textbf{(i) Moments first, gradients second.}
A small budget share releases whitened class-conditional moments once. They define the analytic noise-end field in closed form and route each class's generation path \emph{before} any per-example gradient is consumed.
The rest is spent by DP-SGD under a pre-declared, time-bucketed allocation verified by exact PLD accounting. This planning layer is risk-free but contributes within seed noise: the release carries the utility.

\textbf{(ii) A backbone built for noise.}
Clamping per-token stream norms of a multimodal diffusion transformer (MM-DiT)~\citep{esser2024_sd3} \emph{during pretraining} leaves the noiseless pretraining loss unchanged, yet its private accuracy gain grows monotonically as the noise strengthens in pixel space (Table~\ref{tab:backbone}). The already well-conditioned latent variant does not need it.
Bounded activations empirically concentrate per-sample gradient norms (Appendix~\ref{app:gradstats}): an architectural control for the regime where DP noise is most damaging~\citep{bethune2023_dpsgdclipless}.

\textbf{(iii) One release, two injection ports.}
The same released moments serve two regimes: \emph{distilled} into the weights of a model trained end-to-end, or \emph{injected at sampling time} into a frozen prior whose weights never move.
Which port works is a property of the injected model, and the experiments answer it at both scales: on MNIST the routes are interchangeable, while on frozen SD3-medium only the sampling port survives, beating budget-matched DP-LoRA fine-tuning.

\begin{figure*}[t]
\centering
\includegraphics[width=\textwidth]{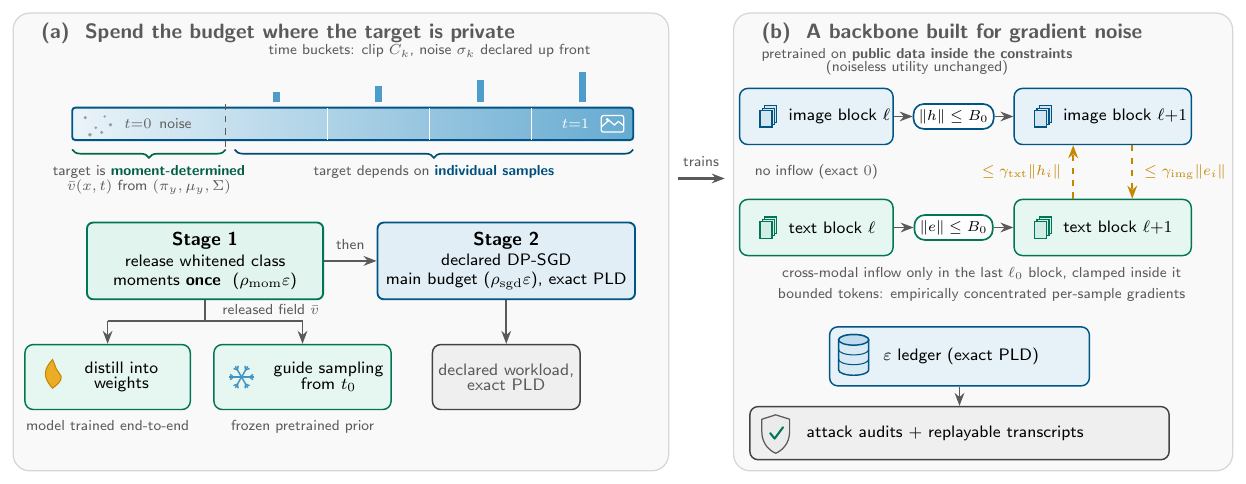}
\caption{StraightDP overview.
(a)~The noise-end velocity target is a closed-form functional of a few class-conditional moments: Stage~1 releases them \emph{once}, and the field enters either the weights (model trained end-to-end) or the sampler (frozen prior). Stage~2 spends the main budget with declared DP-SGD, optionally under a public time-heterogeneous plan.
(b)~The backbone is pretrained on public data inside stream-norm clamps, with cross-modal inflow an architectural zero outside a designated final block and clamped inside it.
Bounded tokens empirically concentrate per-sample gradients exactly where DP noise is strongest, and every run logs audits and a replayable accounting transcript.}
\label{fig:overview}
\end{figure*}

\section{Preliminaries}\label{sec:prelim}

\textbf{Setting.}
A private dataset $\D=\{r_i\}_{i=1}^N$ contains records $r=(z,c)$: an image (or frozen-autoencoder latent) $z\in\R^d$ with $\|z\|_2\le R$, and a caption $c$ encoded by a \emph{frozen} public text encoder; captions carry a class label $y\in[Y]$ (possibly further attributes).
An out-of-domain public dataset $\Dpub$ is available for pretraining and for every calibration.

\textbf{Differential privacy.}
Datasets are \emph{adjacent} if they differ by adding or removing one record, the native semantics of Poisson-subsampled DP-SGD accounting. A mechanism $M$ is $(\epsdp,\deltadp)$-DP if $\Pr[M(\D)\in A]\le e^{\epsdp}\Pr[M(\D')\in A]+\deltadp$ for all adjacent pairs and events.
Every mechanism is composed in these semantics by one numerical PLD accountant, and every run writes a replayable transcript from which the reported $\epsdp$ is recomputed.

\textbf{DP-SGD.}
DP-SGD~\citep{abadi2016_dpsgd} makes one training step private by clipping each example's gradient, $\clip_C(g)=g\cdot\min(1,C/\|g\|)$, and noising the sum over the Poisson-sampled batch: $\tilde g=\sum_{i\in B}\clip_C(g_i)+\N(0,\sigma^2C^2I)$.
The clip norm $C$ bounds one record's influence, the noise multiplier $\sigma$ and the sampling rate set the per-step privacy loss, and the accountant composes the steps.

\textbf{Rectified flow.}
With flow noise $\xi\sim\N(0,I_d)$ and interpolation $x_t=(1-t)\xi+tz$ ($t\in[0,1]$, $t{=}0$ the noise end), the model $v_\theta(x,t,c)$ regresses the target $v=z-\xi$:
\begin{equation}
\label{eq:fmloss}
\mathcal{L}(\theta)=\E_{r,\xi,t}\big\|v_\theta(x_t,t,c)-(z-\xi)\big\|_2^2 .
\end{equation}
Sampling integrates $\dot x = v_\theta$ from $t{=}0$ to $1$ with classifier-free guidance.
The Bayes-optimal field, conditionally on the caption, is $v^\star(x,t,c)=\E[z-\xi\mid x_t=x,\,c]$ (the classifier-free null branch regresses its unconditional counterpart); substituting $\xi=(x-tz)/(1-t)$ gives the identity we use repeatedly, conditionally or not,
\begin{equation}
\label{eq:vstar}
v^\star(x,t)=\frac{\E[z\mid x_t=x]-x}{1-t}.
\end{equation}

The adversary sees final parameters.
Every quantity the pipeline tunes comes from $\Dpub$ or released statistics, so the only private interactions are the declared mechanisms.

\section{The StraightDP Pipeline}\label{sec:tha}

This section takes the intro's three surfaces in turn: Stage~1 releases the moments (\emph{what}), Stage~2 spends the declared main budget (\emph{when}), and the backbone constrains the architecture both stages run on.
Figure~\ref{fig:overview}(a) traces the two stages.

\subsection{Moments first: the analytic low-$t$ field and its release}\label{sec:a1}

\begin{figure}[t]
\centering
\includegraphics[width=0.96\columnwidth]{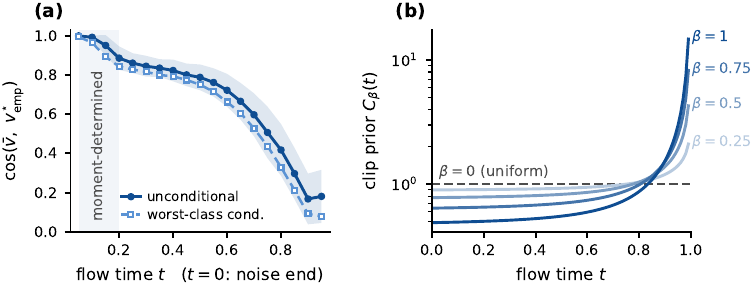}
\caption{The geometry that StraightDP exploits (computed on public data). (a)~Cosine between released and exact empirical fields vs.\ $t$: the unconditional (null-branch) pair \eqref{eq:mixture} and the \emph{worst class} of the conditional pair \eqref{eq:analytic}. Both are moment-dominated over the same low-$t$ horizon. (b)~The clip-prior family $C_\beta(t)$ ($\beta{=}0$ uniform, log scale).}
\label{fig:timegeo}
\end{figure}

\begin{proposition}[Analytic field of a moment approximation]\label{prop:analytic}
Approximate the class-conditional law of $z$ by $\N(\mu_y,\Sigma)$ with class priors $\pi_y$, and let $A_t=(1-t)^2 I + t^2\Sigma$.
Conditionally on the class,
\begin{equation}
\label{eq:analytic}
\E[z\mid x_t{=}x,\,y]=\mu_y+t\Sigma A_t^{-1}(x-t\mu_y),
\end{equation}
while the unconditional field of the classifier-free null branch averages \eqref{eq:analytic} with responsibilities $w_y(x,t)\propto \pi_y\,\N(x;\,t\mu_y,\,A_t)$:
\begin{equation}
\label{eq:mixture}
\E[z\mid x_t{=}x]=\sum_{y} w_y(x,t)\Big(\mu_y+t\Sigma A_t^{-1}(x-t\mu_y)\Big).
\end{equation}
By \eqref{eq:vstar}, as $t\to0$ the conditional field tends to $\mu_y-x$ and the unconditional one to $\sum_y\pi_y\mu_y-x$: each class's noise-end target is its released mean, and the data enter only through $(\pi_y,\mu_y,\Sigma)$.
\end{proposition}

Proof in Appendix~\ref{app:proofs}. In outline, joint Gaussianity of $(z,x_t)$ gives \eqref{eq:analytic} by Gaussian conditioning, and Bayes' rule over the class label turns the class posteriors into the responsibilities of \eqref{eq:mixture}.
The leading order itself is distribution-free:

\begin{corollary}[Distribution-free leading order]\label{cor:smallt}
For \emph{any} class-conditional law supported in the $R$-ball, with mean $\mu_y$ and covariance $\Sigma_y$, uniformly over $\|x\|$ bounded,
\begin{gather*}
\E[z\mid x_t{=}x,\,y]=\mu_y+t\,\Sigma_y x+O(t^2),\\
v^\star(x,t,y)=(\mu_y-x)+t\,(\mu_y-x+\Sigma_y x)+O(t^2).
\end{gather*}
\end{corollary}

Proof in Appendix~\ref{app:proofs}. In outline, the posterior of $z$ given $x_t{=}x$ reweights the prior by $\exp(t\langle x,z\rangle+O(t^2))$ on the bounded support, and a first-order expansion of the ratio of expectations leaves exactly the covariance term.
Appendix~\ref{app:proofs} also closes the loop to sampling: a Gr\"onwall bound propagates any field discrepancy on $t\le\tau$ into a Wasserstein bound on the generated marginal, which is the quantity the public $\tau$ calibration controls, and a Gaussian-tail argument covers the sampler's unbounded inputs.
Moment dominance at the noise end is therefore a property of the interpolation, and the Gaussian model is first-order exact whenever class covariances match the shared $\Sigma$.
Beyond first order the proposition is a \emph{working model} whose reach is not assumed but measured on public data via the horizon $\tau$ below. Figure~\ref{fig:timegeo}a shows the two fields nearly collinear for $t\lesssim0.2$.
Privacy never depends on this adequacy: a poor moment model wastes budget, it does not leak.
Releasing the moments once is substantially cheaper than relearning the same information through per-step noisy gradients.

\textbf{Mechanism (class-conditional moment release).}
The released statistics are deliberately primitive: per-class counts $n_y$ and sums $S_y=\sum_{i:y_i=y}z_i$, plus one shared second moment.
(i) The pair $\{(n_y,S_y)\}_{y\in[Y]}$, all classes stacked, is released through one Gaussian mechanism.
Replacing one record moves two counts by at most $1$ and two sums by at most $R$, so whitening each coordinate by the type-sensitivity vector $(1,2R)$, its type's worst-case move, gives joint $\ell_2$ sensitivity $\Delta=2$ across every block a record can touch (Proposition~\ref{prop:a1sens}).
The class means $\mu_y=S_y/n_y$ are then free post-processing.
(ii) The shared second moment is a second, independent Gaussian mechanism, released in \emph{average} form: one record is a $1/N$ fraction of a mean, so its sensitivity decays as $1/N$, and at $N{=}6\times10^4$ the calibrated noise on the covariance is negligible (Proposition~\ref{prop:a1sens}).
(iii) All remaining geometry is public.
Class-mean offsets are released only inside a $p$-dimensional PCA basis $P\in\R^{p\times d}$ fit on $\Dpub$, which shrinks the noise norm by $\approx\sqrt{d/p}$ at the price of a feature clip $R_{\mathrm{feat}}$, and the shared covariance is kept at a low rank $r$.
These sensitivity bounds are proved for single-record \emph{replacement}. For these additive statistics, normalized by the declared public constant $N$ (held fixed across adjacent datasets), an add/remove change replaces one bounded contribution with zero while every other term keeps its normalization, so the bounds cover the adjacency of Section~\ref{sec:prelim} a fortiori (formalized on padded databases in Appendix~\ref{app:a1sens}).
Noisy counts are clamped at $1$, small classes fall back to the released global mean, the noised second moment is symmetrized and PSD-projected, and the rank-$r$ truncation happens after noising, so all of it is post-processing; both mechanisms enter the same PLD accountant.
Plugging the release into \eqref{eq:analytic} yields the private analytic field $\bar v(x,t,y)$: the noise-end target of every class.

\textbf{Beyond classes: the text-conditional release.}
The finite partition is not essential.
For general captions, Stage~1 instead releases the whitened \emph{cross moments}
\begin{equation}
\label{eq:crossmom}
M_{zc}=\frac1N\sum_i z_i\,e(c_i)^\top,
\qquad
M_{cc}=\frac1N\sum_i e(c_i)e(c_i)^\top,
\end{equation}
with $e(c)$ a caption featurization clipped to a public radius $R_{\mathrm{txt}}$, both blocks coordinate-whitened by the type radii and released through one joint Gaussian mechanism.
In this average form one replacement moves each whitened block by at most $2/N$, so the concatenated release has joint sensitivity $2\sqrt2/N$, the constant the accountant uses.
Ridge post-processing gives the linear conditional mean $\mu(c)=\mu_0+M_{zc}(M_{cc}+\lambda I)^{-1}e(c)$, and Proposition~\ref{prop:analytic} holds verbatim with $\mu_y$ replaced by $\mu(c)$: conditioning only moves the mean.
One-hot features recover the class release \emph{exactly}, so the class-conditional pipeline is the special case (Appendix~\ref{app:textmom}).
The noise geometry governs utility: caption features are centered at their public mean, whose shared component carries no conditional signal yet would dominate the whitened budget, with the removed part absorbed into the already-released $\mu_0$.
The image side reuses the class release's public PCA projection and clip, and templated captions parse into per-slot one-hot blocks.

\textbf{Conditional distillation.}
Stage~1 distills $v_\theta$ toward $\bar v$ on model-generated inputs for $t\le\tau$, conditionally (captions supply $y$, or the embedding $e(c)$ in the text-conditional case, with distillation captions sampled from public data).
The horizon $\tau$ is calibrated \emph{publicly}: we sweep the time-weighted cosine between $\bar v$ and an empirical field on $\Dpub$ and pick the knee, a public-to-private transfer heuristic whose failure costs utility, never privacy (its cross-domain stability is checked in Appendix~\ref{app:proofs}).
Under domain transfer this stage is decisive: without it the model emits public-domain glyphs.

\textbf{Sampling-time guidance.}
The same release can instead be spent at generation, with the weights untouched.
Under the moment model the flow's marginal at time $t_0$ is Gaussian, $x_{t_0}\sim\N\big(t_0\mu_y,\ t_0^2\Sigma+(1-t_0)^2 I\big)$ (with $\mu(c)$ in the text-conditional case), so sampling draws its initial state from this \emph{tilted marginal} and integrates the model from $t_0$ to $1$.
Distillation and guidance thus inject one release through two ports, the weights and the sampler. Section~\ref{sec:sd3} decides between them empirically at both scales.

\subsection{Gradients second: time-heterogeneous allocation (THA) for DP-SGD}\label{sec:thaplan}

The remaining budget goes to what the release cannot cover: the sample-specific target toward the data end.
Stage~2 spends the main budget with DP-SGD under a fully pre-declared workload: steps are partitioned into $K$ time buckets covering $[\tau,1]$, so no private gradient targets the distilled interval $[0,\tau)$, bucket $k$ training on $t\in[t_k,t_{k+1})$ with its own clip $C_k$ and noise multiplier $\sigma_k$, and the exact composition is verified before training.
A public geometry prior can shape $(C_k,\sigma_k)$ across buckets (shape family in Figure~\ref{fig:timegeo}b; planner in Appendix~\ref{app:neverworse}). The reader may take Stage~2 to be uniform DP-SGD with declared accounting.

\begin{remark}[Proxy dominance over uniform]\label{prop:neverworse}
The uniform plan is in the candidate family and all candidates are verified with the same exact accountant at the same budget. Hence the selected plan's accounted $\epsdp$ meets the budget and its proxy score is $\ge$ that of uniform.
This is a selection-consistency statement about the public proxy, not a utility guarantee. Its role is to make the geometry prior available at zero risk.
\end{remark}

The experiments find the plan's contribution within seed noise: the utility comes from the release.

\subsection{A backbone built for noise}\label{sec:backbone}

The third surface is the architecture itself.
Our model is a dual-stream MM-DiT: image tokens $h$ and caption tokens $e$ interact through joint attention in every block~\citep{esser2024_sd3}.
We impose three structural constraints \emph{during public pretraining} and keep them at fine-tuning and generation:

\begin{enumerate}
\item \textbf{Stream-norm clamp.} Between blocks, every token of both streams is projected onto the ball of radius $\Bo$: $h\mapsto h\cdot\min(1,\Bo/\|h\|)$.
\item \textbf{Late injection.} Cross-modal attention inflow is \emph{exactly zero} outside the last $\lo$ blocks: an architectural zero.
\item \textbf{Decoupled attention.} Self- and cross-attention use independent softmax normalizers, so the image-stream update decomposes exactly as $h\mapsto F_\ell(h)+G_\ell(h,e)$ with $F_\ell$ caption-independent; the exact decomposition lets each cross-stream contribution be bounded and ablated separately.
\end{enumerate}

The stream clamp does the utility work: it improves plain DP training with a gain that scales with the noise, while the other two constraints localize and bound cross-modal interaction.
\textbf{The constraints must be active throughout public pretraining}: optimized inside the constraint set from the start, the model matches its unconstrained twin's pretraining loss, while imposing the same targets post hoc on a model pretrained without them collapses it.
The experiments take the three surfaces in order: the release and its allocation, the backbone under noise, and the two injection ports at scale.

\section{Experiments}\label{sec:exp}

\subsection{Experimental setup}\label{sec:setup}
The controlled experiments run on KMNIST$\to$MNIST transfer, with Kuzushiji-MNIST~\citep{clanuwat2018_kmnist} as the public domain $\Dpub$ and MNIST~\citep{lecun1998_mnist} as the private set $\D$ (the domains share no glyphs), at $\deltadp{=}10^{-5}$, in a \emph{pixel-space} and a \emph{latent-space} variant (MM-DiT with $6$ blocks, width $192$). The latent variant works in the $d{=}32$ code of a frozen public convolutional autoencoder.
Every main-text result except Table~\ref{tab:sd3} reports this setting, each configuration under a single $\epsdp$ budget; configurations with the release spend $0.235\epsdp$ on Stage~1 and the rest on DP-SGD (a declared release weight of $0.2$, renormalized over the active stages; the accounting transcripts report the exact split).
Three further studies answer targeted questions: a Fashion-MNIST$\to$MNIST replication~\citep{xiao2017_fashionmnist} checks the gains are not public-domain specific (reported inline below); composed-MNIST extends the release beyond class labels to multi-attribute captions; and DP-LoRA fine-tuning of the frozen 2B-parameter SD3-medium on Flowers-102 tests which components survive at scale (Table~\ref{tab:sd3}).
Utility is downstream accuracy of a classifier trained on generated (image, label) pairs and tested on real data.
Quality uses Fr\'echet inception distance (FID), computed conventionally for comparability (Inception-V3 pool-3 features of grayscale samples replicated to three channels and resized to $299^2$; $10^4$ samples against the full test set) and corroborated by precision/recall in the appendix.
All results are mean$\pm$std over 3 seeds unless noted, and lower-is-better metrics are marked $\downarrow$ throughout.
The two metrics are read jointly: where a configuration trades one axis for the other, both movements are reported and described as a trade, not an improvement.
Published DP diffusion results~\citep{dockhorn2022_dpdm,ghalebikesabi2023_dpdiffusion} are not quoted as baselines: they arise from orders-of-magnitude larger batches, schedules, and evaluation sets, so quoting them would compare compute classes, not mechanisms. Their techniques run inside the protocol instead: noise multiplicity~\citep{dockhorn2022_dpdm} in every DP-SGD configuration and the from-scratch baseline, DP-LDM-style adapter fine-tuning~\citep{lyu2023_dpldm} as the SD3 uniform baseline.
One command reproduces each run (hyperparameters: Appendix~\ref{app:hyper}).

\subsection{Main results: the release drives utility}\label{sec:mainresults}

\begin{table*}[t]
\centering
\footnotesize
\setlength{\tabcolsep}{3.4pt}
\renewcommand{\arraystretch}{0.96}
\caption{Main results (MNIST, $\deltadp{=}10^{-5}$, 3 seeds, mean$\pm$std). Acc = downstream accuracy, $\downarrow$ marks lower-is-better, and bold is the best per column within each pipeline block. The first block runs prior methods under the same protocol and accounting~\citep{harder2021_dpmerf,vinaroz2022_dphp,yang2023_dpntk,lin2024_pe}. The released-moments-only rows spend the \emph{entire} column budget on the Stage-1 mechanism and take no gradient step, a deliberately favorable full-budget control for the release alone; the moments rows spend $0.235\epsdp$ on it. The pixel and latent blocks share the same public (KMNIST) pretraining.}
\label{tab:main}
\begin{tabular}{lcccccc}
\toprule
& \multicolumn{2}{c}{$\epsdp{=}0.3$} & \multicolumn{2}{c}{$\epsdp{=}1$} & \multicolumn{2}{c}{$\epsdp{=}3$} \\
\cmidrule(lr){2-3}\cmidrule(lr){4-5}\cmidrule(lr){6-7}
Method & Acc & FID$\downarrow$ & Acc & FID$\downarrow$ & Acc & FID$\downarrow$ \\
\midrule
\multicolumn{7}{@{}l}{\emph{Prior methods}} \\
DP-MERF (kernel embedding) & $0.823\pm0.004$ & $280.6\pm0.8$ & $0.856\pm0.005$ & $182.8\pm5.0$ & $0.842\pm0.011$ & $109.1\pm2.4$ \\
DP-HP (Hermite embedding) & $0.762\pm0.010$ & $284.3\pm5.0$ & $0.751\pm0.008$ & $219.8\pm4.5$ & $0.780\pm0.006$ & $167.3\pm5.6$ \\
DP-NTK (NTK embedding) & $0.737\pm0.019$ & $333.8\pm3.0$ & $0.786\pm0.012$ & $292.7\pm3.4$ & $0.811\pm0.012$ & $249.1\pm5.7$ \\
Private Evolution (frozen prior) & $0.044\pm0.026$ & $122.4\pm2.1$ & $0.059\pm0.025$ & $81.4\pm0.8$ & $0.082\pm0.022$ & $52.8\pm2.0$ \\
\midrule
\multicolumn{7}{@{}l}{\emph{Released moments only (no private gradient)}} \\
Sampled from the release $\N(\tilde\mu_y,\tilde\Sigma)$ & $0.741\pm0.020$ & $301.0\pm5.2$ & $0.762\pm0.015$ & $236.6\pm1.6$ & $0.797\pm0.020$ & $241.9\pm0.7$ \\
Public prior $+$ guidance & $0.246\pm0.018$ & $97.1\pm1.4$ & $0.236\pm0.026$ & $89.3\pm0.8$ & $0.222\pm0.008$ & $87.0\pm0.4$ \\
\midrule
\multicolumn{7}{@{}l}{\emph{Pixel space}} \\
Uniform DP-SGD & $0.055\pm0.005$ & $\mathbf{72.5\pm0.3}$ & $0.207\pm0.011$ & $\mathbf{62.7\pm0.1}$ & $0.620\pm0.003$ & $51.6\pm0.1$ \\
THA-planned DP-SGD & $0.056\pm0.002$ & $74.5\pm0.5$ & $0.226\pm0.050$ & $63.5\pm1.1$ & $0.641\pm0.086$ & $51.8\pm2.1$ \\
Uniform + moments & $0.484\pm0.024$ & $88.4\pm2.0$ & $0.706\pm0.005$ & $63.6\pm0.7$ & $\mathbf{0.859\pm0.005}$ & $\mathbf{49.3\pm0.5}$ \\
THA + moments & $0.443\pm0.047$ & $94.1\pm2.7$ & $\mathbf{0.724\pm0.033}$ & $66.7\pm0.6$ & $0.848\pm0.004$ & $51.5\pm0.3$ \\
\quad + Low-$\Bo$ backbone & $\mathbf{0.561\pm0.038}$ & $100.3\pm2.6$ & $0.697\pm0.003$ & $75.1\pm1.1$ & $0.788\pm0.004$ & $69.5\pm0.7$ \\
\midrule
\multicolumn{7}{@{}l}{\emph{Latent space (frozen public autoencoder, $d{=}32$)}} \\
Uniform DP-SGD & $0.434\pm0.027$ & $\mathbf{63.3\pm0.6}$ & $0.709\pm0.016$ & $58.4\pm0.4$ & $0.808\pm0.010$ & $51.7\pm0.5$ \\
THA-planned DP-SGD & $0.408\pm0.016$ & $63.3\pm0.3$ & $0.695\pm0.024$ & $59.2\pm0.6$ & $0.798\pm0.013$ & $53.1\pm0.5$ \\
Uniform + moments & $\mathbf{0.695\pm0.015}$ & $63.8\pm1.7$ & $\mathbf{0.814\pm0.002}$ & $\mathbf{55.7\pm0.6}$ & $\mathbf{0.843\pm0.006}$ & $\mathbf{49.6\pm0.3}$ \\
THA + moments & $0.682\pm0.010$ & $64.1\pm1.3$ & $0.811\pm0.002$ & $56.5\pm0.8$ & $0.839\pm0.007$ & $50.2\pm0.1$ \\
\quad + Low-$\Bo$ backbone & $0.633\pm0.012$ & $64.7\pm1.6$ & $0.792\pm0.006$ & $57.6\pm0.2$ & $0.835\pm0.006$ & $51.5\pm0.2$ \\
\bottomrule
\end{tabular}
\end{table*}

Does the released low-$t$ target beat spending the whole budget on gradients, and does the flow model add anything beyond the release? Three observations (Table~\ref{tab:main}, Figure~\ref{fig:epscurve}).
\begin{figure}[t]
\centering
\includegraphics[width=0.57\columnwidth]{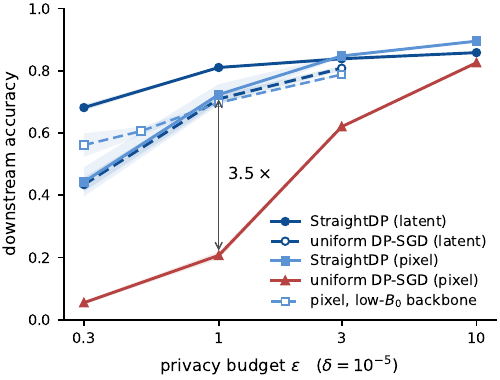}
\caption{MNIST downstream accuracy vs.\ privacy budget $\epsdp$ ($\deltadp{=}10^{-5}$; bands $\pm1$ std over 3 seeds), each representation with its baseline. The low-$\Bo$ variant (dashed, open squares) lifts the strong-privacy end. The bracket marks the $3.5\times$ pixel gap over uniform DP-SGD at $\epsdp{=}1$.}
\label{fig:epscurve}
\end{figure}

\begin{table*}[t]
\centering
\footnotesize
\setlength{\tabcolsep}{2.8pt}
\caption{The low-$\Bo$ backbone, with arrows as in Table~\ref{tab:main} and the bold row the full backbone. \textbf{Left}: single-factor ablation at $\epsdp{=}0.3$ (3 seeds, mean$\pm$std). The unconstrained and full rows are the pixel moments and low-$\Bo$ rows of Table~\ref{tab:main}. \textbf{Right}: the noise-regime sweep on pixels, whose signed gain compares the backbones.}
\label{tab:backbone}
\begin{minipage}[t]{0.435\linewidth}
\centering
\begin{tabular}{@{}lcc@{}}
\toprule
Backbone & Acc & FID$\downarrow$ \\
\midrule
Unconstrained & $0.443\pm0.047$ & $94.1\pm2.7$ \\
+ Late injection only & $0.402\pm0.014$ & $103.0\pm1.2$ \\
+ Stream clamp only & $0.535\pm0.014$ & $98.7\pm1.5$ \\
\textbf{Full backbone} & $0.561\pm0.038$ & $100.3\pm2.6$ \\
\bottomrule
\end{tabular}
\end{minipage}\hfill
\begin{minipage}[t]{0.555\linewidth}
\centering
\begin{tabular}{@{}lccccr@{}}
\toprule
& \multicolumn{2}{c}{Standard} & \multicolumn{2}{c}{Low-$\Bo$} & \\
\cmidrule(lr){2-3}\cmidrule(lr){4-5}
$\epsdp$ & Acc & FID$\downarrow$ & Acc & FID$\downarrow$ & Acc gain \\
\midrule
$3.0$ & $0.848\pm0.004$ & $51.5\pm0.3$ & $0.788\pm0.004$ & $69.5\pm0.7$ & $-7\%$ \\
$1.0$ & $0.724\pm0.033$ & $66.7\pm0.6$ & $0.697\pm0.003$ & $75.1\pm1.1$ & $-4\%$ \\
$0.5$ & $0.584\pm0.032$ & $81.2\pm5.0$ & $0.606\pm0.014$ & $85.7\pm3.9$ & $+4\%$ \\
$0.3$ & $0.443\pm0.047$ & $94.1\pm2.7$ & $0.561\pm0.038$ & $100.3\pm2.6$ & $+27\%$ \\
\bottomrule
\end{tabular}
\end{minipage}
\end{table*}

(1) The released-moments-only block isolates the release, spending each column's whole budget on it.
Sampling it directly, which is equivalent to integrating the analytic field of the released Gaussian moment model to $t{=}1$, already carries strong probe accuracy ($0.74$--$0.80$) with extremely poor label-blind FID ($237$--$301$) and visibly prototype-like samples (Figure~\ref{fig:grids}), the class-prototype effect the embedding family also exhibits.
Guiding the public prior with the release but \emph{no} private gradient remains ineffective on this wrong-domain prior ($0.22$--$0.25$).
Private flow training converts those prototype-like samples into far higher-quality conditional generation while keeping most of the class signal: in pixel space it turns the domain-transfer failure into $0.724$ at $\epsdp{=}1$ (FID $66.7$), and in latent space it improves \emph{both} axes over the release alone, $0.762\to0.811$ accuracy and $236.6\to56.5$ FID.
The THA plan alone moves utility within seed noise, and the uniform~$+$~moments rows match THA~$+$~moments at every budget in both representations: the release does not depend on the plan.
(2) The advantage grows as privacy strengthens ($8\times$ at $\epsdp{=}0.3$ in pixel space; in latent space the representation alone reaches $0.434$ and the release lifts it to $0.682$, a $1.57\times$ gain that shrinks to $1.04\times$ at $\epsdp{=}3$) and persists at weak privacy: pixel at $\epsdp{=}3$ already beats the baseline at $\epsdp{=}10$ ($0.848$ vs.\ $0.826$).
(3) A Fashion-MNIST public domain replicates the gain ($3.0$--$3.2\times$), and utility is insensitive to the Stage-1 share ($0.700/0.724/0.735$ over declared weights $\rho_{\mathrm{mom}}\in\{0.1,0.2,0.3\}$, effective shares $0.12$--$0.35$).
The prior-method block frames these numbers, and no method in the table dominates: the comparison is a Pareto front over accuracy and FID, and every comparison below reports both axes, with improvements named by metric unless both move. The embedding family attains the highest probe accuracies while its FID never enters a usable regime, and neither successor improves on DP-MERF under this uniform protocol. Private Evolution~\citep{lin2024_pe} fails in the opposite direction: its Inception-space votes drive FID to $35.4$ at $\epsdp{=}10$, below every training route, yet accuracy never exceeds chance, since votes only reweight what the frozen prior generates and cannot rebind its class conditioning. 
The sample grids of Figure~\ref{fig:grids} make the disagreement visible, and precision/recall in the same feature space explains it structurally (Appendix~\ref{sec:pr}).
Uniform DP-SGD holds the best \emph{label-blind} scores at every budget precisely because its transfer fails, leaving clean wrong-domain glyphs whose low-level statistics flatter any label-blind metric: a cleanly wrong marginal beats a noisily right one.
Label-aware metrics confirm it: at $\epsdp{=}1$ label consistency is $0.17$ vs.\ our $0.44$, and per-class FID reverses it (Table~\ref{tab:condmetrics}).

The latent variant's low precision is a decoder-smoothing artifact, flat in $\epsdp$ and invisible to the stroke-topology features the downstream classifier reads.
The stream clamp trades recall for gradient signal-to-noise (the $+27\%$ at $\epsdp{=}0.3$).

The release also generalizes beyond class labels: replacing the partition by the cross-moment release \eqref{eq:crossmom} preserves the mechanism and accounting, and the more discrete the featurization, the closer it comes to the class-conditional ceiling: at $\epsdp{=}1$, per-slot indicators reach $0.632\pm0.067$ and the pooled embedding of the frozen text encoder $0.464\pm0.049$, against the one-hot reference $0.724\pm0.033$.
The gap is noise in the released $M_{cc}$, not a mechanism change: it narrows monotonically with budget ($0.345$ vs.\ $0.443$ at $\epsdp{=}0.3$, $0.829$ vs.\ $0.848$ at $\epsdp{=}3$ for indicators), converging to the one-hot special case exactly as the inversion of $M_{cc}+\lambda I$ sharpens.
On composed captions the discrete featurizations condition the non-class attributes essentially perfectly (quadrant $1.00$, stroke $0.84$--$1.00$; Appendix~\ref{app:textmom}).
Both geometric corrections of the release are essential: removing the image-side projection drops accuracy to $0.32$, and removing the centering and released intercept to $0.12$, below the pipeline with no release.

\subsection{The stream-norm constraint under noise}\label{sec:backboneexp}

When does constraining the backbone help, and what does it cost when it does not? The right panel of Table~\ref{tab:backbone} shows the effect \emph{grows monotonically with noise}, $-7\%$, $-4\%$, $+4\%$, $+27\%$ as $\epsdp$ falls from $3$ to $0.3$, lifting strong-privacy accuracy at a FID concession ($94.1\to100.3$) and costing both axes at weak privacy.
The left panel ablates the three constraints at $\epsdp{=}0.3$: the stream clamp carries most of the gain ($0.443\to0.535$ of the full $0.561$), late injection alone is a small net cost ($0.402$), and the full set adds the remainder.
The gain does not transfer to the latent variant, where the constraint is a small net cost at every budget ($0.682\to0.633$ at $\epsdp{=}0.3$ with FID unchanged; Table~\ref{tab:main}), consistent with the mechanism below: the $32$-dimensional latent model's per-sample gradients are already well concentrated, so the clamp only removes capacity.
Mechanistically, bounding activations empirically concentrates per-sample gradient norms, narrowing the upper tail that clipping hits hardest (the $p_{90}/p_{50}$ ratio falls from $1.61$ to $1.46$ at $\epsdp{=}0.3$; Appendix~\ref{app:gradstats}), so DP clipping distorts less exactly when $\sigma$ is large.
Unlike globally Lipschitz networks, whose expressivity cost is documented~\citep{anil2019_lipschitz}, the clamp carries no detectable pretraining-loss penalty, necessary if not sufficient evidence of preserved capacity, while post-hoc clamping of an unconstrained model collapses it to $0.02$.
The recommendation: moments with uniform DP-SGD, in the latent representation when available, with the clamp for extreme-noise pixels.

\subsection{Scaling: the injection route decides}\label{sec:sd3}
Which components survive on frozen SD3-medium?
We DP-LoRA fine-tune SD3-medium on Flowers-102~\citep{nilsback2008_flowers} as the private set ($N{=}2040$, $\epsdp{=}10$, $\deltadp{=}1/2N$), with the base model itself as the public prior: its samples calibrate all clips and projections. DP-SGD and the accounting port unchanged.
Stage~2 uses a single uniform $(C,\sigma)$ so Table~\ref{tab:sd3} isolates the injection route. The guidance route trains no weights, so there is nothing to plan.
The bounded backbone does not port: its constraints are imposed during pretraining, which a frozen checkpoint forecloses.
The base model generates flowers, a \emph{strong-prior} domain: zero-shot reaches $0.588\pm0.009$, and uniform DP-SGD does not beat it ($0.580\pm0.004$).

\begin{table}[tb]
\centering
\footnotesize
\setlength{\tabcolsep}{2.6pt}
\renewcommand{\arraystretch}{0.96}
\caption{SD3-medium on Flowers-102 (3 seeds, mean$\pm$std). Acc = downstream accuracy (CLIP probe), the $\epsdp$ column is each row's total accounted budget, and bold is the best private configuration. Private Evolution~\citep{lin2024_pe} uses the frozen base model. The distillation row writes the moments into LoRA and spends the remainder on DP-SGD ($2{+}8$), while guidance injects the \emph{same} moments at sampling time.}
\label{tab:sd3}
\begin{tabular}{lccc}
\toprule
Method & $\epsdp$ & Acc & FID$\downarrow$ \\
\midrule
Zero-shot (no tuning) & $0$ & $0.588\pm0.009$ & $40.2\pm0.6$ \\
\midrule
Uniform DP-SGD & $10$ & $0.580\pm0.004$ & $37.9\pm0.6$ \\
Uniform DP-SGD & $2$ & $0.603\pm0.007$ & $38.6\pm0.6$ \\
Private Evolution & $10$ & $0.593\pm0.015$ & $43.5\pm0.6$ \\
Distillation $+$ DP-SGD & $10$ & $0.088\pm0.018$ & $105.7\pm1.4$ \\
\textbf{Moment guidance} & $\mathbf{2}$ & $\mathbf{0.625\pm0.005}$ & $\mathbf{37.8\pm0.5}$ \\
DP-SGD $+$ guidance & $10.4$ & $0.614\pm0.013$ & $38.0\pm1.0$ \\
\midrule
Non-private LoRA & $\infty$ & $0.657\pm0.019$ & $35.8\pm1.0$ \\
\quad $+$ guidance & $\infty$ & $0.678\pm0.013$ & $35.4\pm0.6$ \\
\bottomrule
\end{tabular}
\end{table}

\textbf{On frozen SD3, the released moments belong in sampling, not in the weights.}
The global first-moment release scales (at $N{=}2040$ per-class signal-to-noise is too low, so Stage~1 falls back to a global mean tilt). What fails is \emph{distilling} it into weights.
Distilling the analytic field into LoRA is catastrophic ($0.088$) even though the field is released accurately and the loss converges: the Gaussian is cruder than the frozen prior it overwrites, and the damage leaks across all noise levels through the shared adapter.
SD3 parameterizes the flow by the noise level $\sigma=1-t$ ($\sigma{=}1$ the noise end; not the DP noise multiplier). Guidance is pure post-processing of the \emph{same} release: sampling starts at its tilted marginal at level $\sigma_0$ and the base model integrates the rest. This reaches $0.625\pm0.005$ while consuming only $\epsdp{=}2$ in total.
The start level is selection-insensitive, stable across $\sigma_0\in[0.90,0.98]$, and the $\sigma_0{=}0.95$ we use is the grid midpoint, so no private tuning is implied.
The budget-matched control confirms the advantage is not a budget artifact: uniform DP-SGD at $\epsdp{=}2$ reaches $0.603\pm0.007$, above its $\epsdp{=}10$ counterpart ($0.580$) because the lighter fine-tuning damages the prior less, yet below guidance at the same budget: the ranking $0.580,0.588,0.603,0.625$ tracks how little the weights move.
Specificity controls isolate the released direction as the payload: starting at the same $\sigma_0$ and covariance with a zero mean or a norm-matched random tilt falls below zero-shot ($0.548\pm0.016$ and $0.555\pm0.021$, FID $\approx64$), the public-prior mean recovers part ($0.606\pm0.007$), and the released direction reaches $0.625$--$0.632$ across $\epsdp\in\{0.5,1,4\}$, matched by exact non-private moments ($0.624\pm0.006$): the payload is the direction, not its precision.
Private Evolution, the other training-free route, evolves a per-class population by noisy nearest-neighbor votes.
With the full $\epsdp{=}10$ it moves the frozen prior only within seed noise ($0.593\pm0.015$ vs.\ zero-shot $0.588$) and trails guidance on both metrics at $5\times$ the budget (FID $43.5$, degraded by small-population resampling). Votes only reweight what the prior generates (hence PE's collapse on weak-prior MNIST, Table~\ref{tab:main}). The moments supply the geometry it lacks.
The embedding family~\citep{harder2021_dpmerf} has no published variant at this resolution (its comparison is Table~\ref{tab:main}).
Adding guidance to training yields little more ($0.614$ at joint $\epsdp{=}10.4$; $0.678$ at $\epsdp{=}\infty$): both approach the same ceiling, and the tilt already captures the recoverable signal.
\begin{figure}[t]
\centering
\includegraphics[width=0.60\columnwidth]{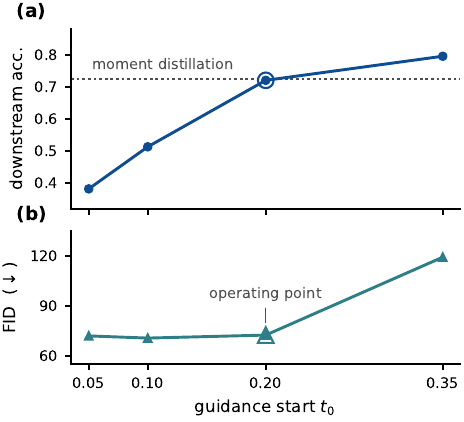}
\caption{The sampling port on MNIST: a DP-SGD model trained \emph{without} moment distillation, guided at sampling time by the Stage-1 release alone.
The sweep is a post-hoc diagnostic at total budget $1.2$ ($\epsdp{=}1$ training $+$ $0.2$ release, single seed). The $t_0{=}0.2$ we use is the public plateau-edge choice, and the budget-matched $0.8{+}0.2$ control gives $0.696\pm0.005$. (a)~Accuracy crosses moment distillation's level (dotted) at $t_0{=}0.2$. (b)~FID is flat there and degrades beyond it.}
\label{fig:guidancet0}
\end{figure}

The tilt helps the non-private model by a similar margin, so its DP-specific value is not the shift itself but that it captures most of the private signal \emph{without any private training}: at $\epsdp{=}2$ it beats every trained configuration at $5\times$ the budget.

\textbf{The two injection ports carry the same release at two scales} (class-conditional on MNIST, the global tilt on SD3).
The start time is calibrated publicly: on the public domain, guidance accuracy is flat for $t_0\le0.2$ and public FID degrades sharply beyond it, so we use the plateau edge $t_0{=}0.2$.
On MNIST, guiding a DP-SGD model with no distillation reaches $0.696\pm0.005$ at matched $\epsdp{=}1$, near distillation's $0.724$ (Figure~\ref{fig:guidancet0}).
On frozen SD3-medium, distillation collapses ($0.088$) and only guidance works ($0.625$).

The suggested split: distill into a model trained end-to-end, guide a strong frozen prior.

\subsection{Attack audits}\label{sec:auditsmain}
Accounted budgets are worst-case statements. Three attacks probe the final models (attack failure bounds nothing from above; protocols in Appendix~\ref{sec:audits}).
Membership inference on the flow-matching residual is chance-level (AUC $0.482$). None of $50$ canaries is preferentially generated, and a one-run audit is vacuous.
Transcript replay reproduces every declared $\epsdp$, and attack advantages are consistent with the accounted guarantees.

\section{Conclusion}\label{sec:conclusion}
StraightDP starts from one observation: what a rectified flow learns from private data is not uniform in flow time.
From it, the pipeline releases class-conditional moments once, spends the rest through declared DP-SGD, and pretrains a stream-bounded backbone for the strongest-noise regime.
One release enters an end-to-end model's weights and a frozen prior's sampler.
The training-free baselines we test land on one axis, the latent variant moves both, and per-caption release at scale is the open frontier.

\newpage

\bibliography{references}

@inproceedings{abadi2016_dpsgd,
  title={Deep Learning with Differential Privacy},
  author={Martín Abadi and Andy Chu and Ian Goodfellow and H. Brendan McMahan and Ilya Mironov and Kunal Talwar and Li Zhang},
  booktitle={Proceedings of the 2016 ACM SIGSAC Conference on Computer and Communications Security (CCS)},
  year={2016}
}

@inproceedings{albergo2022_interpolants,
  title={Building Normalizing Flows with Stochastic Interpolants},
  author={Michael S. Albergo and Eric Vanden-Eijnden},
  booktitle={International Conference on Learning Representations (ICLR)},
  year={2023}
}

@inproceedings{andrew2021_adaptiveclipping,
  title={Differentially Private Learning with Adaptive Clipping},
  author={Galen Andrew and Om Thakkar and H. Brendan McMahan and Swaroop Ramaswamy},
  booktitle={Advances in Neural Information Processing Systems 34},
  year={2021}
}

@inproceedings{anil2019_lipschitz,
  title={Sorting out Lipschitz function approximation},
  author={Cem Anil and James Lucas and Roger Grosse},
  booktitle={Proceedings of the 36th International Conference on Machine Learning (ICML)},
  year={2019}
}

@inproceedings{balle2018_subsampling,
  title={Privacy Amplification by Subsampling: Tight Analyses via Couplings and Divergences},
  author={Borja Balle and Gilles Barthe and Marco Gaboardi},
  booktitle={Advances in Neural Information Processing Systems 31},
  year={2018}
}

@inproceedings{bethune2023_dpsgdclipless,
  title={DP-SGD Without Clipping: The Lipschitz Neural Network Way},
  author={Louis Béthune and Thomas Massena and Thibaut Boissin and Yannick Prudent and Corentin Friedrich and Franck Mamalet and Aurélien Bellet and Mathieu Serrurier and David Vigouroux},
  booktitle={International Conference on Learning Representations (ICLR)},
  year={2024}
}

@article{bie2023_privategans,
  title={Private GANs, Revisited},
  author={Alex Bie and Gautam Kamath and Guojun Zhang},
  journal={Transactions on Machine Learning Research},
  year={2023}
}

@inproceedings{bu2022_autoclip,
  title={Automatic Clipping: Differentially Private Deep Learning Made Easier and Stronger},
  author={Zhiqi Bu and Yu-Xiang Wang and Sheng Zha and George Karypis},
  booktitle={Advances in Neural Information Processing Systems 36},
  year={2023}
}

@inproceedings{bun2016_zcdp,
  title={Concentrated Differential Privacy: Simplifications, Extensions, and Lower Bounds},
  author={Mark Bun and Thomas Steinke},
  booktitle={Theory of Cryptography Conference (TCC)},
  year={2016}
}

@inproceedings{cao2021_dpsinkhorn,
  title={Don't Generate Me: Training Differentially Private Generative Models with Sinkhorn Divergence},
  author={Tianshi Cao and Alex Bie and Arash Vahdat and Sanja Fidler and Karsten Kreis},
  booktitle={Advances in Neural Information Processing Systems 34},
  year={2021}
}

@inproceedings{carlini2019_secretsharer,
  title={The Secret Sharer: Evaluating and Testing Unintended Memorization in Neural Networks},
  author={Nicholas Carlini and Chang Liu and Úlfar Erlingsson and Jernej Kos and Dawn Song},
  booktitle={28th USENIX Security Symposium},
  year={2019}
}

@inproceedings{carlini2022_miafp,
  title={Membership Inference Attacks From First Principles},
  author={Nicholas Carlini and Steve Chien and Milad Nasr and Shuang Song and Andreas Terzis and Florian Tramer},
  booktitle={2022 IEEE Symposium on Security and Privacy (S\&P)},
  year={2022}
}

@inproceedings{chen2020_gswgan,
  title={GS-WGAN: A Gradient-Sanitized Approach for Learning Differentially Private Generators},
  author={Dingfan Chen and Tribhuvanesh Orekondy and Mario Fritz},
  booktitle={Advances in Neural Information Processing Systems 33},
  year={2020}
}

@article{clanuwat2018_kmnist,
  title={Deep Learning for Classical Japanese Literature},
  author={Tarin Clanuwat and Mikel Bober-Irizar and Asanobu Kitamoto and Alex Lamb and Kazuaki Yamamoto and David Ha},
  journal={arXiv preprint arXiv:1812.01718},
  year={2018}
}

@article{de2022_unlocking,
  title={Unlocking High-Accuracy Differentially Private Image Classification through Scale},
  author={Soham De and Leonard Berrada and Jamie Hayes and Samuel L. Smith and Borja Balle},
  journal={arXiv preprint arXiv:2204.13650},
  year={2022}
}

@article{dockhorn2022_dpdm,
  title={Differentially Private Diffusion Models},
  author={Tim Dockhorn and Tianshi Cao and Arash Vahdat and Karsten Kreis},
  journal={Transactions on Machine Learning Research},
  year={2023}
}

@article{dong2019_gdp,
  title={Gaussian Differential Privacy},
  author={Jinshuo Dong and Aaron Roth and Weijie J. Su},
  journal={Journal of the Royal Statistical Society Series B},
  volume={84}, number={1}, pages={3--37},
  year={2022}
}

@article{doroshenko2022_connectdots,
  title={Connect the Dots: Tighter Discrete Approximations of Privacy Loss Distributions},
  author={Vadym Doroshenko and Badih Ghazi and Pritish Kamath and Ravi Kumar and Pasin Manurangsi},
  journal={arXiv preprint arXiv:2207.04380},
  year={2022}
}

@inproceedings{esser2024_sd3,
  title={Scaling Rectified Flow Transformers for High-Resolution Image Synthesis},
  author={Patrick Esser and Sumith Kulal and Andreas Blattmann and Rahim Entezari and Jonas Müller and Harry Saini and Yam Levi and Dominik Lorenz and Axel Sauer and Frederic Boesel and Dustin Podell and Tim Dockhorn and Zion English and Kyle Lacey and Alex Goodwin and Yannik Marek and Robin Rombach},
  booktitle={Proceedings of the 41st International Conference on Machine Learning (ICML)},
  year={2024}
}

@inproceedings{ganesh2023_whypublic,
  title={Why Is Public Pretraining Necessary for Private Model Training?},
  author={Arun Ganesh and Mahdi Haghifam and Milad Nasr and Sewoong Oh and Thomas Steinke and Om Thakkar and Abhradeep Thakurta and Lun Wang},
  booktitle={Proceedings of the 40th International Conference on Machine Learning (ICML)},
  year={2023}
}

@article{ghalebikesabi2023_dpdiffusion,
  title={Differentially Private Diffusion Models Generate Useful Synthetic Images},
  author={Sahra Ghalebikesabi and Leonard Berrada and Sven Gowal and Ira Ktena and Robert Stanforth and Jamie Hayes and Soham De and Samuel L. Smith and Olivia Wiles and Borja Balle},
  journal={arXiv preprint arXiv:2302.13861},
  year={2023}
}

@inproceedings{gopi2021_numerical,
  title={Numerical Composition of Differential Privacy},
  author={Sivakanth Gopi and Yin Tat Lee and Lukas Wutschitz},
  booktitle={Advances in Neural Information Processing Systems 34},
  year={2021}
}

@inproceedings{harder2021_dpmerf,
  title={DP-MERF: Differentially Private Mean Embeddings with Random Features for Practical Privacy-Preserving Data Generation},
  author={Frederik Harder and Kamil Adamczewski and Mijung Park},
  booktitle={Proceedings of the 24th International Conference on Artificial Intelligence and Statistics (AISTATS)},
  year={2021}
}

@inproceedings{vinaroz2022_dphp,
  title={Hermite Polynomial Features for Private Data Generation},
  author={Margarita Vinaroz and Mohammad-Amin Charusaie and Frederik Harder and Kamil Adamczewski and Mi Jung Park},
  booktitle={Proceedings of the 39th International Conference on Machine Learning (ICML)},
  year={2022}
}

@article{yang2023_dpntk,
  title={Differentially Private Neural Tangent Kernels for Privacy-Preserving Data Generation},
  author={Yilin Yang and Kamil Adamczewski and Danica J. Sutherland and Xiaoxiao Li and Mijung Park},
  journal={arXiv preprint arXiv:2303.01687},
  year={2023}
}

@inproceedings{lin2024_pe,
  title={Differentially Private Synthetic Data via Foundation Model {APIs} 1: Images},
  author={Zinan Lin and Sivakanth Gopi and Janardhan Kulkarni and Harsha Nori and Sergey Yekhanin},
  booktitle={International Conference on Learning Representations (ICLR)},
  year={2024}
}

@inproceedings{heusel2017_fid,
  title={GANs Trained by a Two Time-Scale Update Rule Converge to a Local Nash Equilibrium},
  author={Martin Heusel and Hubert Ramsauer and Thomas Unterthiner and Bernhard Nessler and Sepp Hochreiter},
  booktitle={Advances in Neural Information Processing Systems 30},
  year={2017}
}

@inproceedings{jagielski2020_auditing,
  title={Auditing Differentially Private Machine Learning: How Private is Private SGD?},
  author={Matthew Jagielski and Jonathan Ullman and Alina Oprea},
  booktitle={Advances in Neural Information Processing Systems 33},
  year={2020}
}

@inproceedings{kairouz2015_composition,
  title={The Composition Theorem for Differential Privacy},
  author={Peter Kairouz and Sewoong Oh and Pramod Viswanath},
  booktitle={Proceedings of the 32nd International Conference on Machine Learning (ICML)},
  year={2015}
}

@inproceedings{kim2021_lipschitzattn,
  title={The Lipschitz Constant of Self-Attention},
  author={Hyunjik Kim and George Papamakarios and Andriy Mnih},
  booktitle={Proceedings of the 38th International Conference on Machine Learning (ICML)},
  year={2021}
}

@inproceedings{koskela2020_fft,
  title={Computing Tight Differential Privacy Guarantees Using FFT},
  author={Antti Koskela and Joonas Jälkö and Antti Honkela},
  booktitle={Proceedings of the 23rd International Conference on Artificial Intelligence and Statistics (AISTATS)},
  year={2020}
}

@article{kurakin2022_imagenetdp,
  title={Toward Training at ImageNet Scale with Differential Privacy},
  author={Alexey Kurakin and Shuang Song and Steve Chien and Roxana Geambasu and Andreas Terzis and Abhradeep Thakurta},
  journal={arXiv preprint arXiv:2201.12328},
  year={2022}
}

@inproceedings{kynkaanniemi2019_pr,
  title={Improved Precision and Recall Metric for Assessing Generative Models},
  author={Tuomas Kynkäänniemi and Tero Karras and Samuli Laine and Jaakko Lehtinen and Timo Aila},
  booktitle={Advances in Neural Information Processing Systems 32},
  year={2019}
}

@inproceedings{lee2024_improvingrf,
  title={Improving the Training of Rectified Flows},
  author={Sangyun Lee and Zinan Lin and Giulia Fanti},
  booktitle={Advances in Neural Information Processing Systems 37},
  year={2024}
}

@inproceedings{lipman2022_flowmatching,
  title={Flow Matching for Generative Modeling},
  author={Yaron Lipman and Ricky T. Q. Chen and Heli Ben-Hamu and Maximilian Nickel and Matt Le},
  booktitle={International Conference on Learning Representations (ICLR)},
  year={2023}
}

@inproceedings{liu2022_rectifiedflow,
  title={Flow Straight and Fast: Learning to Generate and Transfer Data with Rectified Flow},
  author={Xingchao Liu and Chengyue Gong and Qiang Liu},
  booktitle={International Conference on Learning Representations (ICLR)},
  year={2023}
}

@article{liu2022_rfot,
  title={Rectified Flow: A Marginal Preserving Approach to Optimal Transport},
  author={Qiang Liu},
  journal={arXiv preprint arXiv:2209.14577},
  year={2022}
}

@inproceedings{long2019_gpate,
  title={G-PATE: Scalable Differentially Private Data Generator via Private Aggregation of Teacher Discriminators},
  author={Yunhui Long and Boxin Wang and Zhuolin Yang and Bhavya Kailkhura and Aston Zhang and Carl A. Gunter and Bo Li},
  booktitle={Advances in Neural Information Processing Systems 34},
  year={2021}
}

@article{lyu2023_dpldm,
  title={DP-LDMs: Differentially Private Latent Diffusion Models},
  author={Saiyue Lyu and Michael F. Liu and Margarita Vinaroz and Mijung Park},
  journal={Transactions on Machine Learning Research},
  year={2024}
}

@inproceedings{mcmahan2018_recurrent,
  title={Learning Differentially Private Recurrent Language Models},
  author={H. Brendan McMahan and Daniel Ramage and Kunal Talwar and Li Zhang},
  booktitle={International Conference on Learning Representations (ICLR)},
  year={2018}
}

@article{mehta2022_transferdp,
  title={Large Scale Transfer Learning for Differentially Private Image Classification},
  author={Harsh Mehta and Abhradeep Thakurta and Alexey Kurakin and Ashok Cutkosky},
  journal={arXiv preprint arXiv:2205.02973},
  year={2022}
}

@inproceedings{mironov2017_rdp,
  title={Rényi Differential Privacy},
  author={Ilya Mironov},
  booktitle={2017 IEEE 30th Computer Security Foundations Symposium (CSF)},
  year={2017}
}

@article{mironov2019_sgm,
  title={Rényi Differential Privacy of the Sampled Gaussian Mechanism},
  author={Ilya Mironov and Kunal Talwar and Li Zhang},
  journal={arXiv preprint arXiv:1908.10530},
  year={2019}
}

@inproceedings{miyato2018_spectralnorm,
  title={Spectral Normalization for Generative Adversarial Networks},
  author={Takeru Miyato and Toshiki Kataoka and Masanori Koyama and Yuichi Yoshida},
  booktitle={International Conference on Learning Representations (ICLR)},
  year={2018}
}

@inproceedings{naeem2020_densitycoverage,
  title={Reliable Fidelity and Diversity Metrics for Generative Models},
  author={Muhammad Ferjad Naeem and Seong Joon Oh and Youngjung Uh and Yunjey Choi and Jaejun Yoo},
  booktitle={Proceedings of the 37th International Conference on Machine Learning (ICML)},
  year={2020}
}

@inproceedings{nasr2021_adversary,
  title={Adversary Instantiation: Lower Bounds for Differentially Private Machine Learning},
  author={Milad Nasr and Shuang Song and Abhradeep Thakurta and Nicolas Papernot and Nicholas Carlini},
  booktitle={2021 IEEE Symposium on Security and Privacy (S\&P)},
  year={2021}
}

@inproceedings{parmar2022_cleanfid,
  title={On Aliased Resizing and Surprising Subtleties in GAN Evaluation},
  author={Gaurav Parmar and Richard Zhang and Jun-Yan Zhu},
  booktitle={Proceedings of the IEEE/CVF Conference on Computer Vision and Pattern Recognition (CVPR)},
  year={2022}
}

@inproceedings{peebles2022_dit,
  title={Scalable Diffusion Models with Transformers},
  author={William Peebles and Saining Xie},
  booktitle={Proceedings of the IEEE/CVF International Conference on Computer Vision (ICCV)},
  year={2023}
}

@article{pichapati2019_adaclip,
  title={AdaCliP: Adaptive Clipping for Private SGD},
  author={Venkatadheeraj Pichapati and Ananda Theertha Suresh and Felix X. Yu and Sashank J. Reddi and Sanjiv Kumar},
  journal={arXiv preprint arXiv:1908.07643},
  year={2019}
}

@inproceedings{sajjadi2018_pr,
  title={Assessing Generative Models via Precision and Recall},
  author={Mehdi S. M. Sajjadi and Olivier Bachem and Mario Lucic and Olivier Bousquet and Sylvain Gelly},
  booktitle={Advances in Neural Information Processing Systems 31},
  year={2018}
}

@inproceedings{shokri2017_mia,
  title={Membership Inference Attacks against Machine Learning Models},
  author={Reza Shokri and Marco Stronati and Congzheng Song and Vitaly Shmatikov},
  booktitle={2017 IEEE Symposium on Security and Privacy (S\&P)},
  year={2017}
}

@inproceedings{steinke2023_onerun,
  title={Privacy Auditing with One (1) Training Run},
  author={Thomas Steinke and Milad Nasr and Matthew Jagielski},
  booktitle={Advances in Neural Information Processing Systems 36},
  year={2023}
}

@inproceedings{tramer2021_features,
  title={Differentially Private Learning Needs Better Features (or Much More Data)},
  author={Florian Tramèr and Dan Boneh},
  booktitle={International Conference on Learning Representations (ICLR)},
  year={2021}
}

@inproceedings{wang2019_subsampledrdp,
  title={Subsampled Rényi Differential Privacy and Analytical Moments Accountant},
  author={Yu-Xiang Wang and Borja Balle and Shiva Prasad Kasiviswanathan},
  booktitle={Proceedings of the 22nd International Conference on Artificial Intelligence and Statistics (AISTATS)},
  year={2019}
}

@article{xiao2017_fashionmnist,
  title={Fashion-MNIST: a Novel Image Dataset for Benchmarking Machine Learning Algorithms},
  author={Han Xiao and Kashif Rasul and Roland Vollgraf},
  journal={arXiv preprint arXiv:1708.07747},
  year={2017}
}

@article{xie2018_dpgan,
  title={Differentially Private Generative Adversarial Network},
  author={Liyang Xie and Kaixiang Lin and Shu Wang and Fei Wang and Jiayu Zhou},
  journal={arXiv preprint arXiv:1802.06739},
  year={2018}
}

@inproceedings{yeom2018_privacyrisk,
  title={Privacy Risk in Machine Learning: Analyzing the Connection to Overfitting},
  author={Samuel Yeom and Irene Giacomelli and Matt Fredrikson and Somesh Jha},
  booktitle={2018 IEEE 31st Computer Security Foundations Symposium (CSF)},
  year={2018}
}

@inproceedings{yu2022_dpfinetune,
  title={Differentially Private Fine-tuning of Language Models},
  author={Da Yu and Saurabh Naik and Arturs Backurs and Sivakanth Gopi and Huseyin A. Inan and Gautam Kamath and Janardhan Kulkarni and Yin Tat Lee and Andre Manoel and Lukas Wutschitz and Sergey Yekhanin and Huishuai Zhang},
  booktitle={International Conference on Learning Representations (ICLR)},
  year={2022}
}

@inproceedings{jordon2019_pategan,
  title={{PATE-GAN}: Generating Synthetic Data with Differential Privacy Guarantees},
  author={Jordon, James and Yoon, Jinsung and van der Schaar, Mihaela},
  booktitle={International Conference on Learning Representations (ICLR)},
  year={2019}
}

@article{sommer2019_privacylossclasses,
  title={Privacy Loss Classes: The Central Limit Theorem in Differential Privacy},
  author={Sommer, David M. and Meiser, Sebastian and Mohammadi, Esfandiar},
  journal={Proceedings on Privacy Enhancing Technologies},
  volume={2019}, number={2}, pages={245--269},
  year={2019}
}

@article{dwork2014_book,
  title={The Algorithmic Foundations of Differential Privacy},
  author={Dwork, Cynthia and Roth, Aaron},
  journal={Foundations and Trends in Theoretical Computer Science},
  volume={9}, number={3--4}, pages={211--407},
  year={2014}
}

@article{du2021_dynamicdp,
  title={Dynamic Differential-Privacy Preserving SGD},
  author={Jian Du and Song Li and Xiangyi Chen and Siheng Chen and Mingyi Hong},
  journal={arXiv preprint arXiv:2111.00173},
  year={2021}
}

@article{lecun1998_mnist,
  title={Gradient-Based Learning Applied to Document Recognition},
  author={Yann LeCun and L{\'e}on Bottou and Yoshua Bengio and Patrick Haffner},
  journal={Proceedings of the IEEE},
  volume={86},
  number={11},
  pages={2278--2324},
  year={1998}
}

@inproceedings{nilsback2008_flowers,
  title={Automated Flower Classification over a Large Number of Classes},
  author={Maria-Elena Nilsback and Andrew Zisserman},
  booktitle={Indian Conference on Computer Vision, Graphics and Image Processing},
  year={2008}
}

@article{tan2023privacy,
  title={Privacy Amplification for Wireless Federated Learning with R{\'e}nyi Differential Privacy and Subsampling},
  author={Qingjie Tan and Xujun Che and Shuhui Wu and Yaguan Qian and Yuanhong Tao},
  journal={Electronic Research Archive},
  volume={31},
  number={11},
  pages={7021},
  year={2023}
}

\newpage
\appendix

\section{Extended Related Work}\label{app:related}

\textbf{DP accounting and DP-SGD.}
DP-SGD~\citep{abadi2016_dpsgd} clips each example's gradient to a norm bound $C$, adds Gaussian noise calibrated to $C$, and tracks the privacy loss of the resulting subsampled Gaussian mechanism across steps.
Its analysis has been refined along two axes.
The per-step mechanism is understood through privacy amplification by subsampling~\citep{balle2018_subsampling,tan2023privacy} and tight R\'enyi bounds for the sampled Gaussian~\citep{mironov2019_sgm,wang2019_subsampledrdp}, within the R\'enyi-DP framework~\citep{mironov2017_rdp} and its relatives zCDP~\citep{bun2016_zcdp} and $f$-DP/GDP~\citep{dong2019_gdp}.
Composition across steps moved from generic bounds~\citep{kairouz2015_composition} to numerically exact ones: privacy-loss distributions~\citep{sommer2019_privacylossclasses} composed by FFT~\citep{koskela2020_fft,gopi2021_numerical}, with optimal discretizations in~\citet{doroshenko2022_connectdots}.
We use this numerical machinery as a \emph{verification} layer: every candidate plan, homogeneous or not, is composed exactly before it runs.
The line of work closest in motivation is adaptive clipping, which tunes the clip norm during training, per layer~\citep{mcmahan2018_recurrent}, by privately tracked quantiles~\citep{andrew2021_adaptiveclipping}, per coordinate~\citep{pichapati2019_adaclip}, or by normalization that removes the clip threshold altogether~\citep{bu2022_autoclip}; \citet{du2021_dynamicdp} schedule the mechanism dynamically across steps.
All of these react to \emph{private} signal, so the adaptation itself either consumes budget or complicates the accounting.
Our time-heterogeneous allocation makes the opposite choice: the entire plan $(C_k,\sigma_k,T_k)$ is computed from a public geometry prior and declared before training, so the accountant verifies a fixed workload and the plan costs nothing beyond it.

\textbf{DP generative models.}
The GAN generation attacked the problem through the discriminator: DP-GAN~\citep{xie2018_dpgan} noises discriminator gradients, PATE-GAN~\citep{jordon2019_pategan} replaces the discriminator with a private teacher ensemble, G-PATE~\citep{long2019_gpate} aggregates gradients from teacher discriminators, GS-WGAN~\citep{chen2020_gswgan} sanitizes only the gradients that reach the generator, and~\citet{bie2023_privategans} show that careful training recipes recover much of the lost utility.
A second family avoids iterative training altogether and releases a statistic once: DP-MERF~\citep{harder2021_dpmerf} a random-feature kernel mean embedding, DP-Sinkhorn~\citep{cao2021_dpsinkhorn} a debiased optimal-transport divergence; later members replace the random Fourier features by deterministic Hermite polynomial features (DP-HP~\citep{vinaroz2022_dphp}) or by the empirical neural-tangent-kernel features of a random network (DP-NTK~\citep{yang2023_dpntk}). A third line trains nothing at all: Private Evolution~\citep{lin2024_pe} evolves a synthetic population from a frozen foundation model; private data acts only through noisy nearest-neighbor vote histograms, and the entire budget is spent at sampling time. We run DP-MERF, DP-HP, DP-NTK, and Private Evolution under the paper's protocol (Tables~\ref{tab:main} and~\ref{tab:sd3}); Section~\ref{sec:quality} analyzes why the embedding family and the evolution mechanism fail on opposite metrics.
Stage~1 of StraightDP shares the one-shot-release form of the embedding family; what differs is \emph{which} statistic is released, the class-conditional moments that determine the analytic noise-end field of a rectified flow rather than a generator-agnostic kernel embedding, and \emph{where} it enters, as a distillation target inside a privately trained flow or as a sampling-time initialization of a frozen prior, so the same release also composes with gradient training instead of replacing it.
Our Stage~1 belongs to this family in mechanism, with one difference we consider decisive: the released statistic is not a generic embedding but the exact sufficient statistics of the flow's low-$t$ Bayes target, so every released coordinate is one the training objective provably consumes.
The diffusion generation trains score or velocity models under DP-SGD: DPDM~\citep{dockhorn2022_dpdm} introduces noise multiplicity, averaging the loss over several diffusion times per example to cut gradient variance at no privacy cost, which we adopt in Stage~2; \citet{ghalebikesabi2023_dpdiffusion} scale public-pretrain-then-DP-fine-tune to strong image quality; DP-LDM~\citep{lyu2023_dpldm} fine-tunes only the conditioning modules of a frozen latent diffusion model, which motivates our latent variant with a frozen public autoencoder.
Across all three families the privacy mechanism is uniform in diffusion time; the structural lever StraightDP turns, \emph{when} the budget is spent and on \emph{what} supervision, is unused.

\textbf{Rectified flows and multimodal DiTs.}
Rectified flow~\citep{liu2022_rectifiedflow} and flow matching~\citep{lipman2022_flowmatching} train continuous-time generative models by regressing a velocity field on straight (or affine Gaussian) interpolation paths, simulation-free; stochastic interpolants~\citep{albergo2022_interpolants} give the general framework, \citet{liu2022_rfot} the optimal-transport reading, and \citet{lee2024_improvingrf} training refinements.
Straightness matters to us beyond sampling speed: on the straight path the posterior of the clean sample given $x_t$ is available in closed form under a moment model, which is what makes the low-$t$ target an explicit functional of $(\pi_y,\mu_y,\Sigma)$ and hence releasable at $1/N$ sensitivity.
On the architecture side, DiT~\citep{peebles2022_dit} established transformer backbones for diffusion, and SD3's MM-DiT~\citep{esser2024_sd3} runs separate image and text token streams coupled by joint attention.
That separation is essential for us: per-stream norms and per-interface inflows are well defined objects to clamp.

\textbf{Constrained networks for privacy.}
Spectral normalization~\citep{miyato2018_spectralnorm} constrains per-layer operator norms for GAN stability; Lipschitz architectures~\citep{anil2019_lipschitz} make the global constant exact at a documented expressivity cost; and standard dot-product attention is not Lipschitz at all, which \citet{kim2021_lipschitzattn} repair with an $L2$ variant.
\citet{bethune2023_dpsgdclipless} connect this line to privacy: a globally Lipschitz network has bounded per-sample gradients, so DP-SGD can run without clipping.
We take the same lever, enforced norms, to a different place.
Instead of a global Lipschitz constant we bound \emph{activations} (per-token stream norms) everywhere, sidestepping the global expressivity cost; attention never needs a Lipschitz constant because the clamps bound its inputs and its inflow directly.
The empirical finding that the activation bound alone improves plain DP training, with a gain that grows with the noise, has no counterpart in that line.

\textbf{Public data and audits.}
That private learning at tight budgets needs public structure was argued early through features~\citep{tramer2021_features} and is now standard practice at scale~\citep{de2022_unlocking,yu2022_dpfinetune,mehta2022_transferdp,kurakin2022_imagenetdp}, with~\citet{ganesh2023_whypublic} giving a theoretical account of why pretraining is hard to avoid.
StraightDP follows this practice and pushes it one step further: beyond the weights, every calibration (clip references, bucket shapes, the distillation horizon, enforcement targets) comes from public data, so no private budget is spent on tuning.
Our audit battery instantiates the standard attack hierarchy: membership inference by shadow statistics~\citep{shokri2017_mia} and by loss thresholds~\citep{yeom2018_privacyrisk}, refined to likelihood-ratio form in~\citet{carlini2022_miafp}; canary extraction~\citep{carlini2019_secretsharer}; and quantitative audits that turn attack success into empirical $\epsdp$ lower bounds~\citep{jagielski2020_auditing,nasr2021_adversary}, in the one-run form of~\citet{steinke2023_onerun}.
Generation quality uses FID~\citep{heusel2017_fid} with the implementation cautions of~\citet{parmar2022_cleanfid}, and precision/recall in the improved form of~\citet{kynkaanniemi2019_pr}, descended from~\citet{sajjadi2018_pr} and paralleled by density/coverage~\citep{naeem2020_densitycoverage}; Sections~\ref{sec:quality} and~\ref{sec:pr} analyze why these marginal metrics and conditional utility can move in opposite directions.

\section{Proofs}\label{app:proofs}

\subsection{Conventions and standing facts}\label{app:conventions}

\textbf{Notation.}
$\|\cdot\|$ denotes the Euclidean norm on vectors; for matrices, $\|M\|:=\smax(M)=\sup_{\|v\|=1}\|Mv\|$ is the operator (spectral) norm and $\|M\|_F$ the Frobenius norm.
A ``token'' is one row of a stream tensor; per-token statements are uniform over tokens.
Primed quantities ($z'$, $e'$, $g'$, \dots) always refer to the neighboring dataset under the adjacency being analyzed; every input not explicitly primed is held fixed, including all shared randomness of the training step (minibatch selection, flow noise $\xi$, time $t$); sensitivity bounds must hold conditionally on this randomness, and all of ours do.
\textbf{Adjacency.} The paper accounts under add/remove adjacency (the native semantics of the subsampled-Gaussian PLD).
All sensitivity bounds in this appendix are proved for single-record \emph{replacement}, which dominates: a removal (or addition) is the special case of replacing a record by an empty contribution, so every replacement bound is also an add/remove bound; released averages use the public constant $N$ as denominator, so the neighboring dataset's size never enters a mechanism.
Every step uses enforced quantities or exact identities.

We isolate four elementary facts that the proofs use repeatedly, so that each later argument can cite them by name rather than re-deriving them inline.

\begin{fact}[Gaussian conditioning]\label{fact:gauss}
Let $(a,b)$ be jointly Gaussian with means $(\mu_a,\mu_b)$, $\mathrm{Cov}(b)=\Sigma_{bb}\succ0$ and $\mathrm{Cov}(a,b)=\Sigma_{ab}$.
Then $\E[a\mid b]=\mu_a+\Sigma_{ab}\Sigma_{bb}^{-1}(b-\mu_b)$.
\end{fact}

\begin{fact}[Projections are nonexpansive]\label{fact:proj}
Let $\Omega\subset\R^d$ be closed and convex and $\Pi_\Omega$ the Euclidean projection onto $\Omega$.
Then $\|\Pi_\Omega(u)-\Pi_\Omega(v)\|\le\|u-v\|$ for all $u,v$.
In particular the per-token stream clamp $\Pi_{\Bo}(v)=v\cdot\min(1,\Bo/\|v\|)$, which is the projection onto the centered ball of radius $\Bo$, is $1$-Lipschitz.
\end{fact}

\begin{proof}
The projection onto a closed convex set satisfies the variational inequality $\langle u-\Pi_\Omega(u),\,w-\Pi_\Omega(u)\rangle\le0$ for all $w\in\Omega$.
Applying it with $w=\Pi_\Omega(v)$, and symmetrically with the roles of $u,v$ exchanged, and adding the two inequalities gives
$\|\Pi_\Omega(u)-\Pi_\Omega(v)\|^2\le\langle u-v,\,\Pi_\Omega(u)-\Pi_\Omega(v)\rangle$,
and Cauchy--Schwarz finishes.
That $\Pi_{\Bo}$ is this projection is immediate: for $\|v\|\le\Bo$ it is the identity, and for $\|v\|>\Bo$ the closest point of the ball lies on the segment from the origin to $v$ at radius $\Bo$.
\end{proof}

\begin{fact}[Operator-norm toolbox]\label{fact:opnorm}
For matrices of compatible shapes:
(i) $\|Mv\|\le\|M\|\,\|v\|$ and $\|MN\|\le\|M\|\,\|N\|$;
(ii) $\smax$ is $1$-Lipschitz with respect to the spectral norm: $|\smax(M)-\smax(N)|\le\|M-N\|$ (Weyl);
(iii) $\smax(cM)=|c|\,\smax(M)$ for $c\ge0$ (positive homogeneity);
(iv) $\smax$ is convex: $\smax(\lambda M+(1-\lambda)N)\le\lambda\,\smax(M)+(1-\lambda)\,\smax(N)$ for $\lambda\in[0,1]$.
\end{fact}

\begin{proof}
(i) and (iii) are immediate from the definition as a supremum.
For (iv), write $\smax(M)=\sup_{\|u\|=\|v\|=1}u^\top Mv$: a pointwise supremum of functions that are \emph{linear} in $M$, hence convex.
(ii) follows from (iv) and homogeneity via the triangle inequality $\smax(M)\le\smax(N)+\smax(M-N)$ applied in both directions.
\end{proof}

\begin{fact}[Row-stochastic averaging does not amplify norms]\label{fact:rowstoch}
Let $A\in\R^{n\times m}$ have nonnegative entries with $\sum_jA_{ij}=1$ for every row $i$, and let $v_1,\dots,v_m$ be vectors.
Then $\big\|\sum_jA_{ij}v_j\big\|\le\max_j\|v_j\|$ for every $i$.
\end{fact}

\begin{proof}
$\sum_jA_{ij}v_j$ is a convex combination of the $v_j$, hence lies in their convex hull; the norm, being convex, attains its maximum over a convex hull at an extreme point, i.e., at one of the $v_j$.
(Directly: $\|\sum_jA_{ij}v_j\|\le\sum_jA_{ij}\|v_j\|\le\max_j\|v_j\|\sum_jA_{ij}=\max_j\|v_j\|$.)
Softmax attention rows are exactly of this form.
\end{proof}

\subsection{Proof of Proposition~\ref{prop:analytic} (analytic field)}\label{app:analytic}

We prove the assertions in turn: the conditional closed form \eqref{eq:analytic}, its mixture average \eqref{eq:mixture}, the velocity identity \eqref{eq:vstar}, and the noise-end limits.

\textbf{Step 1 (single Gaussian: the joint law of $(z,x_t)$).}
Fix the class $y$ and let $z\sim\N(\mu_y,\Sigma)$ and $\xi\sim\N(0,I_d)$ be independent, $x_t=(1-t)\xi+tz$.
The pair $(z,x_t)$ is the image of the jointly Gaussian vector $(z,\xi)$ under a fixed linear map, hence itself jointly Gaussian; it therefore suffices to compute first and second moments.
By linearity of expectation, $\E[x_t]=(1-t)\E[\xi]+t\E[z]=t\mu_y$.
By bilinearity of covariance and independence of $z$ and $\xi$ (so that every cross term $\mathrm{Cov}(z,\xi)$ vanishes),
\begin{align*}
\mathrm{Cov}(x_t)
&=\mathrm{Cov}\big((1-t)\xi\big)+\mathrm{Cov}\big(tz\big)\\
&\quad+\underbrace{(1-t)t\,\big[\mathrm{Cov}(\xi,z)+\mathrm{Cov}(z,\xi)\big]}_{=\,0}\\
&=(1-t)^2 I+t^2\Sigma \;=:\; A_t,\\[1mm]
\mathrm{Cov}(z,x_t)
&=\mathrm{Cov}\big(z,(1-t)\xi\big)+\mathrm{Cov}\big(z,tz\big)\\
&=0+t\,\Sigma\;=\;t\Sigma .
\end{align*}
For every $t\in[0,1)$ we have $(1-t)^2>0$, so $A_t\succeq(1-t)^2I\succ0$: $A_t$ is invertible wherever the formula below uses it.

\textbf{Step 2 (single Gaussian: conditional mean).}
Fact~\ref{fact:gauss} applied to $(a,b)=(z,x_t)$ gives
\begin{equation}
\label{eq:condmean}
\E[z\mid x_t=x,\,y]\;=\;\mu_y+t\Sigma A_t^{-1}(x-t\mu_y),
\end{equation}
which is \eqref{eq:analytic}.

\textbf{Step 3 (mixture: posterior responsibilities).}
Now let $y\sim\pi$ and $z\mid y\sim\N(\mu_y,\Sigma)$.
Conditionally on $y$, Step 1 shows $x_t\mid y\sim\N(t\mu_y,A_t)$ (note that $A_t$ does not depend on $y$ because $\Sigma$ is shared; this is the benefit of the shared covariance).
Bayes' rule for densities gives the posterior class weights
\[
w_y(x,t)\;:=\;\Pr[y\mid x_t=x]
\;=\;\frac{\pi_y\,\N(x;\,t\mu_y,\,A_t)}{\sum_{y'}\pi_{y'}\,\N(x;\,t\mu_{y'},\,A_t)} ,
\]
and the tower property of conditional expectation combines them with \eqref{eq:condmean}:
\begin{align*}
\E[z\mid x_t=x]
&=\sum_y \Pr[y\mid x_t=x]\;\E[z\mid x_t=x,\,y]\\
&=\sum_y w_y(x,t)\Big(\mu_y+t\Sigma A_t^{-1}(x-t\mu_y)\Big),
\end{align*}
which is exactly \eqref{eq:mixture}.

\textbf{Step 4 (from conditional mean to velocity).}
On the event $\{x_t=x\}$, the interpolation identity can be solved for the noise: $\xi=(x-tz)/(1-t)$ holds pointwise (deterministically given $z$ and $x$).
Substituting,
\[
z-\xi
\;=\;z-\frac{x-tz}{1-t}
\;=\;\frac{(1-t)z-x+tz}{1-t}
\;=\;\frac{z-x}{1-t}.
\]
Both sides are integrable, so taking $\E[\,\cdot\mid x_t=x]$ and using linearity yields
\[
v^\star(x,t)=\E[z-\xi\mid x_t=x]=\frac{\E[z\mid x_t=x]-x}{1-t},
\]
which is \eqref{eq:vstar}; conditioning on $y$ throughout gives the same identity for the conditional field.

\textbf{Step 5 (noise-end limit).}
Fix a compact set of $x$ and let $t\to0$.
First, $A_t-I=(t^2-2t)I+t^2\Sigma$, so $\|A_t-I\|\le 2t+t^2(1+\|\Sigma\|)$; hence $A_t\to I$ and, by the Neumann series, $A_t^{-1}=I+O(t)$ uniformly.
Hence the correction term in \eqref{eq:condmean} satisfies $\|t\Sigma A_t^{-1}(x-t\mu_y)\|=O(t)$ uniformly on compacts.
Second, the exponent of $\N(x;t\mu_y,A_t)$ depends on $y$ only through the term $t\mu_y$, so the ratio of any two class densities is $1+O(t)$, and therefore $w_y(x,t)=\pi_y+O(t)$.
Substituting both expansions,
\begin{align*}
\E[z\mid x_t=x]&=\sum_y\pi_y\mu_y+O(t),\\
v^\star(x,t)&=\sum_y\pi_y\mu_y-x+O(t):
\end{align*}
to leading order the target depends on the data \emph{only} through $(\pi_y,\mu_y,\Sigma)$, and the class structure enters at order $t$ through the responsibilities.
\qed

\begin{remark}[Relation to the empirical kernel field]
Taking the degenerate limit $\Sigma\to h^2I$ with one mixture component per training point ($\mu_i=z_i$, $\pi_i=1/N$) turns \eqref{eq:mixture} into the empirical kernel field with Gaussian weights $\propto\exp\big(-\|x-tz_i\|^2/(2((1-t)^2+t^2h^2))\big)$ at bandwidth $h$: Proposition~\ref{prop:analytic} and the exact empirical target are two points on one family, which is why Figure~\ref{fig:timegeo} (left) can compare them directly.
\end{remark}

\begin{lemma}[Noise-end expansion for a general law]\label{lem:smallt}
Let the class-conditional law of $z$ given $y$ be \emph{any} distribution supported in $\{\|z\|\le R\}$ with mean $\mu_y$ and covariance $\Sigma_y$.
Then, uniformly over $\|x\|$ bounded and $t\le\tfrac12$,
\begin{gather*}
\E[z\mid x_t{=}x,\,y]=\mu_y+t\,\Sigma_y x+O(t^2),\\
v^\star(x,t,y)=(\mu_y-x)+t\,(\mu_y-x+\Sigma_y x)+O(t^2).
\end{gather*}
\end{lemma}

\begin{proof}
By Bayes' rule the posterior of $z$ given $x_t=x$ reweights the prior by the interpolation kernel:
dropping $z$-free factors of $\N(x;tz,(1-t)^2I)$,
\[
w_t(z)=\exp\!\Big(\frac{2t\langle x,z\rangle-t^2\|z\|^2}{2(1-t)^2}\Big)
=\exp\big(t\langle x,z\rangle+r_t(z)\big),
\]
where, on $\|z\|\le R$, $\|x\|\le X$, $t\le\tfrac12$, the remainder obeys $|r_t(z)|\le C(R,X)\,t^2$ (the $(1-t)^{-2}$ expansion and the $t^2\|z\|^2$ term are both $O(t^2)$ with bounded coefficients).
Hence $w_t(z)=1+t\langle x,z\rangle+O(t^2)$ uniformly, and with $\E$ the class-conditional expectation,
\begin{align*}
\E[z\mid x_t{=}x,y]&=\frac{\E[z\,w_t(z)]}{\E[w_t(z)]}\\
&=\big(\mu_y+t\,\E[z\langle x,z\rangle]\big)\big(1-t\langle x,\mu_y\rangle\big)+O(t^2)\\
&=\mu_y+t\big(\E[zz^\top]-\mu_y\mu_y^\top\big)x+O(t^2),
\end{align*}
which is $\mu_y+t\Sigma_y x+O(t^2)$.
The field form follows from \eqref{eq:vstar}: with $(1-t)^{-1}=1+t+O(t^2)$,
\begin{align*}
v^\star(x,t,y)&=\big(\mu_y-x+t\Sigma_y x+O(t^2)\big)\big(1+t+O(t^2)\big)\\
&=(\mu_y-x)+t\,(\mu_y-x+\Sigma_y x)+O(t^2).
\end{align*}
The Gaussian working model \eqref{eq:analytic} expands to $\mu_y+t\Sigma x+O(t^2)$, so it is first-order exact for every law whose class covariance equals the shared $\Sigma$; class-dependent covariances enter first at order $t$ through $\Sigma_y-\Sigma$.
Moment dominance at the noise end is therefore a property of the interpolation, not of the Gaussian approximation.
\end{proof}

\begin{proposition}[Field-error propagation over the horizon]\label{prop:groenwall}
Let $x$ solve $\dot x=\bar v(x,t)$ and $x^\star$ solve $\dot x^\star=v^\star(x^\star,t)$ from the same initial point on $[0,\tau]$, where $v^\star(\cdot,t)$ is $L$-Lipschitz on a region containing both trajectories and $\|\bar v(x,t)-v^\star(x,t)\|\le\delta(t)$ there.
Then
\[
\|x(\tau)-x^\star(\tau)\|
\;\le\;\int_0^\tau e^{L(\tau-s)}\,\delta(s)\,ds
\;\le\;\frac{e^{L\tau}-1}{L}\,\sup_{s\le\tau}\delta(s),
\]
and coupling the two flows by a shared initial draw bounds the Wasserstein-2 distance of the time-$\tau$ marginals by the same quantity.
\end{proposition}

\begin{proof}
With $e(s)=\|x(s)-x^\star(s)\|$,
$\dot e\le\|\bar v(x,s)-v^\star(x,s)\|+\|v^\star(x,s)-v^\star(x^\star,s)\|\le\delta(s)+L\,e(s)$,
and Gr\"onwall's inequality gives the integral bound; the supremum form follows by monotonicity.
For the marginals, the shared-draw coupling realizes a transport plan whose cost is bounded by the pathwise bound.
\end{proof}

This is the quantity the public $\tau$ calibration controls: the knee criterion keeps the measured field discrepancy $\delta(t)$ small on $t\le\tau$, and at the calibrated $\tau\le0.35$ the exponential factor is benign, so a small measured $\delta$ certifies a small sampling-error contribution from the moment segment.
One gap is stated explicitly: the calibration measures $\delta$ on $\Dpub$, whereas the bound needs the discrepancy on the private domain, so carrying it over requires a transfer assumption of the form
$\sup_{t\le\tau}\|v^\star_{\mathrm{priv}}-\bar v_{\mathrm{priv}}\|\le\sup_{t\le\tau}\|v^\star_{\mathrm{pub}}-\bar v_{\mathrm{pub}}\|+\Delta_{\mathrm{shift}}$.
We do not bound $\Delta_{\mathrm{shift}}$; we support it empirically, by the stability of the selected knee across two disjoint public domains (KMNIST and Fashion-MNIST select the same $\tau$ in the latent variant; Section~\ref{app:hyper}), and structurally, by the fact that a mis-transferred $\tau$ wastes budget but never leaks: privacy is independent of every calibration.

\begin{remark}[Unbounded sampler inputs]\label{rem:unboundedx}
Lemma~\ref{lem:smallt} is uniform over $\|x\|\le X$, while the sampler's input $x_t=(1-t)\xi+tz$ has unbounded Gaussian $\xi$.
Since $\|x_t\|\le(1-t)\|\xi\|+tR$ and $\Pr[\|\xi\|>\sqrt d+u]\le e^{-u^2/2}$, the lemma applies along the sampling path with $X=\sqrt d+tR+u$ except on an event of probability $e^{-u^2/2}$; the constant $C(R,X)$ grows only linearly in $X$, so the expansion holds with high probability over the sampler's own randomness.
\end{remark}

\subsection{Sensitivity of the moment release}\label{app:a1sens}

We first recall the adjacency and the exact object being released, then prove the two sensitivity bounds separately.

\textbf{Setup.}
Datasets $\D=\{(z_i,y_i)\}_{i=1}^N$ and $\D'$ are \emph{replace-adjacent} if they agree except that one record $(z,y)\in\D$ is replaced by $(z',y')\in\D'$; in particular $|\D|=|\D'|=N$, and all records satisfy $\|z_i\|\le R$.
The Stage-1 mechanism releases (a) the whitened count/sum vector $u(\D)$ defined below with isotropic Gaussian noise, and (b) the average-form second moment with its own Gaussian noise; class-mean offsets are post-processed from (a) inside the public PCA basis.

\textbf{Add/remove adjacency on padded databases.}
The add/remove semantics of the main text reduces to the replacement semantics above.
Extend the record domain by a null record $\bot$ contributing $\phi(\bot)=0$ to every released statistic of the form $\tfrac1N\sum_i\phi(r_i)$, and represent a dataset as a capacity-$N$ tuple over the extended domain, where $N$ is the declared public constant, identical for all adjacent pairs.
Adding a record is the replacement $\bot\to r$ and removal is $r\to\bot$; either changes exactly one summand, by at most $\sup_r\|\phi(r)\|/N$, and leaves the normalization of every other term unchanged (the accountant never uses the realized dataset size).
The replacement bounds proved below allow arbitrary pairs $r\to r'$, so they dominate both directions, and the Stage-1 transcripts compose with the Poisson-subsampled DP-SGD transcript in this same add/remove semantics under one PLD accountant.

\begin{proposition}[Release sensitivities]\label{prop:a1sens}
Define the whitened statistic vector
\begin{gather*}
u(\D)\;=\;\Big(n_y,\ \tfrac{1}{2R}\,S_y\Big)_{y\in[Y]}\in\R^{Y(1+d)},\\
n_y=|\{i:y_i=y\}|,\qquad S_y=\textstyle\sum_{i:y_i=y}z_i .
\end{gather*}
Under replace adjacency,
$\sup_{\D\sim\D'}\|u(\D)-u(\D')\|_2=\sqrt{5/2}\le 2$,
where $2$ is the value declared to the accountant.
Moreover the second moment $M_2(\D)=\frac1N\sum_i z_iz_i^\top$ satisfies
$\sup_{\D\sim\D'}\|M_2(\D)-M_2(\D')\|_F\le 2R^2/N$.
\end{proposition}

\begin{proof}
Write $u=u(\D)$, $u'=u(\D')$; only the blocks of the affected classes can differ.

\emph{Case 1: label unchanged ($y'=y$).}
Then $n$-coordinates are all unchanged, and only class $y$'s whitened-sum block moves:
\[
\|u-u'\|_2
=\frac{1}{2R}\|z'-z\|
\;\le\;\frac{\|z'\|+\|z\|}{2R}
\;\le\;\frac{R+R}{2R}
\;=\;1 .
\]

\emph{Case 2: label changed ($y'\ne y$).}
Exactly the four sub-blocks $(n_y,\,S_y/2R,\,n_{y'},\,S_{y'}/2R)$ move, and they occupy \emph{disjoint} coordinates of $u$, so their squared changes add:
\begin{align*}
n_y&:\ \ n_y\mapsto n_y-1, &&\text{change }1;\\
n_{y'}&:\ \ n_{y'}\mapsto n_{y'}+1, &&\text{change }1;\\
S_y/2R&:\ \ \text{loses } z/2R, &&\text{change }\le \|z\|/2R\le\tfrac12;\\
S_{y'}/2R&:\ \ \text{gains } z'/2R, &&\text{change }\le \|z'\|/2R\le\tfrac12 .
\end{align*}
Therefore
\[
\|u-u'\|_2^2\;\le\;1^2+1^2+\big(\tfrac12\big)^2+\big(\tfrac12\big)^2\;=\;\tfrac52 ,
\]
with equality attained by $\|z\|=\|z'\|=R$ and orthogonal directions, so the exact worst case over both cases is $\sqrt{5/2}\approx1.58$.
Declaring $\Delta=2$ to the accountant is therefore sound (and mildly conservative; we keep the round constant for robustness to implementation drift).

\emph{Second moment.}
The shared terms cancel:
\begin{gather*}
M_2(\D)-M_2(\D')=\frac1N\big(zz^\top-z'z'^{\top}\big),\\
\|zz^\top-z'z'^{\top}\|_F\le\|z\|^2+\|z'\|^2\le2R^2,
\end{gather*}
using $\|vv^\top\|_F=\|v\|^2$ for rank-one matrices.
Dividing by $N$ gives the claim; this is the $\propto1/N$ decay the Stage-1 release exploits.

\emph{Post-processing steps.}
Two later operations act on the released statistics and cannot increase sensitivity:
projecting mean offsets onto the public PCA basis $P\in\R^{p\times d}$ satisfies $\|Pv\|\le\|v\|$ because $P$ has orthonormal rows ($PP^\top=I_p$, so $\|Pv\|^2=v^\top P^\top P v\le\|v\|^2$ since $P^\top P$ is an orthogonal projector); and the feature clip $\|\cdot\|\le R_{\mathrm{feat}}$ is a projection onto a ball, nonexpansive by Fact~\ref{fact:proj}.
Both are functions of released (already-noised) quantities and public matrices, hence pure post-processing for the accountant.
\end{proof}

\subsection{Cross-moment release: sensitivity and class equivalence}\label{app:crossmom}

\begin{proposition}\label{prop:crossmom}
Let $u(\D)=\big(\tfrac{1}{RR_{\mathrm{txt}}}M_{zc},\ \tfrac{1}{R_{\mathrm{txt}}^{2}}M_{cc}\big)$ with $M_{zc}=\tfrac1N\sum_i z_ie_i^\top$, $M_{cc}=\tfrac1N\sum_i e_ie_i^\top$, $\|z_i\|\le R$, $\|e_i\|\le R_{\mathrm{txt}}$.
Then, under replace adjacency (and a fortiori under image-only and caption-only adjacency),
$\sup_{\D\sim\D'}\|u(\D)-u(\D')\|_2\le 2\sqrt2/N$.
Moreover, with one-hot embeddings $e_i=\mathrm{onehot}(y_i)$ and ridge $\lambda\to0$, the regression map $W=M_{zc}(M_{cc}+\lambda I)^{-1}$ satisfies $W_{:,y}=\bar z_y$, the empirical mean of class $y$.
\end{proposition}

\begin{proof}
\emph{Sensitivity.}
Replacing $(z,e)$ by $(z',e')$ changes the first block by
\[
\frac{\|ze^\top-z'e'^{\top}\|_F}{N\,RR_{\mathrm{txt}}}
\;\le\;\frac{\|z\|\|e\|+\|z'\|\|e'\|}{N\,RR_{\mathrm{txt}}}
\;\le\;\frac{2}{N},
\]
using $\|ab^\top\|_F=\|a\|\,\|b\|$ for rank-one matrices and the triangle inequality; the second block is bounded by $2/N$ identically with $z$ replaced by $e$.
The two blocks occupy disjoint coordinates, so the joint change is at most $\sqrt{(2/N)^2+(2/N)^2}=2\sqrt2/N$.
Image-only ($e'=e$) and caption-only ($z'=z$) replacements are special cases of the pair replacement, so the same bound holds under either single-field replacement.

\emph{Class equivalence.}
With one-hot embeddings, $M_{cc}=\mathrm{diag}(\hat\pi)$ with $\hat\pi_y=n_y/N$ and $M_{zc}=[\hat\pi_1\bar z_1,\dots,\hat\pi_Y\bar z_Y]$ (column $y$ sums the class-$y$ images divided by $N$).
Hence $W_{:,y}=\hat\pi_y\bar z_y/(\hat\pi_y+\lambda)\to\bar z_y$ as $\lambda\to0$: the text-conditional release reduces exactly to the per-class means, and Proposition~\ref{prop:analytic} applies with $\mu(c)$ in place of $\mu_y$ because the Gaussian-conditioning argument of Section~\ref{app:analytic} never uses more than the conditional mean.
\end{proof}

\subsection{The THA planner and Remark~\ref{prop:neverworse}}\label{app:neverworse}
This section specifies the allocation planner referenced in the main text: the shape family $C_\beta(t)\propto\big(w(t)\sqrt{\bar V(t)}/(1-t)\big)^{\beta}$ built from the loss weighting $w(t)$ and the posterior-variance trace $\bar V(t)=\operatorname{tr}\mathrm{Cov}[z\mid x_t]$ of the released analytic field, with $\beta\in\{0,\tfrac14,\tfrac12,\tfrac34,1\}$ and $\beta{=}0$ the uniform plan.

We first formalize the planner, then verify its two guarantees; both are immediate once stated precisely: the guarantee holds by construction rather than by any property of the prior.

\textbf{The planner as an algorithm.}
Fix the budget $(\epsdp_{\mathrm{budget}},\deltadp)$, the bucket structure, and the public prior functions $w(\cdot),\bar V(\cdot)$.
For each $\beta$ in the finite family $\mathcal{B}=\{0,\tfrac14,\tfrac12,\tfrac34,1\}$:
\begin{enumerate}
\item form the clip profile $C_\beta(t)\propto\big(w(t)\sqrt{\bar V(t)}/(1-t)\big)^{\beta}$, discretized to the buckets (note $\beta=0$ gives the constant profile: the \emph{uniform} plan);
\item calibrate a single noise scale by bisection so that the \emph{exact} PLD accountant, run on the fully declared workload (bucket counts, sampling rate, per-bucket $C_k,\sigma_k$), certifies $\epsdp(P_\beta)\le\epsdp_{\mathrm{budget}}$; discard $\beta$ if calibration fails in the search range;
\item score the calibrated candidate with a public utility proxy $U(P_\beta)$.
\end{enumerate}
Let $\mathcal{B}'\subseteq\mathcal{B}$ be the calibrated subset and return $P^\star=\arg\max_{\beta\in\mathcal{B}'}U(P_\beta)$.

\textbf{Guarantee (i): budget.}
$P^\star\in\mathcal{B}'$, and membership in $\mathcal{B}'$ \emph{is} the accountant's certificate $\epsdp(P^\star)\le\epsdp_{\mathrm{budget}}$; nothing further is needed.

\textbf{Guarantee (ii): never worse than uniform.}
The uniform plan $P_0$ always calibrates (bisection over a single scale for a fixed workload always terminates in the standard range), so $P_0\in\mathcal{B}'$; since $P^\star$ maximizes $U$ over a set containing $P_0$, we get $U(P^\star)\ge U(P_0)$.
This holds for \emph{every} choice of proxy $U$: the guarantee does not depend on the prior being right, only on uniform being in the family.

\textbf{Privacy cost of planning: zero.}
$w$, $\bar V$ and $U$ are computed on public data; the accountant is a deterministic function of declared parameters.
The private data is touched only by the single plan that is eventually executed, whose privacy is exactly the certificate from step 2.
\qed

\section{Text-conditional moment release}\label{app:textmom}

The finite partition is not essential.
For general captions, release the whitened \emph{cross moments} $\big(\tfrac{1}{RR_{\mathrm{txt}}}M_{zc},\ \tfrac{1}{R_{\mathrm{txt}}^2}M_{cc}\big)$ with $M_{zc}=\tfrac1N\sum_i z_i e(c_i)^\top$ and $M_{cc}=\tfrac1N\sum_i e(c_i)e(c_i)^\top$, where $e(c)$ is a caption featurization clipped to $R_{\mathrm{txt}}$ (publicly calibrated).
One replacement moves each whitened block by at most $2/N$ (rank-one differences), so the joint replace sensitivity is $2\sqrt2/N$, the same $1/N$ decay as the covariance. The bound is moreover unchanged when only the image or only the caption of a record is replaced (Proposition~\ref{prop:crossmom}).
Ridge post-processing yields the linear conditional mean $\mu(c)=\mu_0+M_{zc}(M_{cc}+\lambda I)^{-1}e(c)$, and Proposition~\ref{prop:analytic} holds verbatim with $\mu_y$ replaced by $\mu(c)$: conditioning only moves the mean.
With one-hot embeddings this \emph{exactly} recovers the per-class means (Proposition~\ref{prop:crossmom}), so the class-conditional release is the special case.
The mechanism is fixed; what governs utility is the \emph{noise geometry} of the featurization, and three choices matter.
(i) \emph{Centering and intercept}: caption embeddings share a large common component that carries no conditional information but dominates the whitened budget, so we center $e(c)$ at its public mean and absorb the shared part into the released unconditional mean $\mu_0$ (the global release of Stage~1, already paid for).
(ii) \emph{Image-side projection}: as in the class-conditional release, the cross moment is formed against features $P(z-\mu_0)$ clipped to $R_{\mathrm{feat}}$ in the public PCA basis, and $W$ is lifted back through $P^\top$.
(iii) \emph{Featurization}: for templated captions we parse against the public template and use per-slot indicator vectors (one one-hot block per attribute); for free-form captions, the pooled embedding of a frozen public text encoder.
The three featurizations span discrete to free-form conditioning: one-hot vectors (classes), per-slot indicators (attribute compositions), and pooled encoder embeddings (unrestricted captions). Table~\ref{tab:textmom} compares them.

\begin{table*}[t]
\centering
\footnotesize
\caption{Text-conditional moment release at $\epsdp{=}1$: MNIST downstream accuracy and per-attribute conditioning accuracy on composed-MNIST (captions specify class, quadrant marker, and stroke thickness), across caption featurizations; 3 seeds throughout. The class-conditional release is the exact one-hot special case (Proposition~\ref{prop:crossmom}). The last two rows ablate the Stage-1 release geometry (single seed, MNIST only).}
\label{tab:textmom}
\begin{tabular}{lcccc}
\toprule
 & MNIST & \multicolumn{3}{c}{composed-MNIST} \\
\cmidrule(lr){2-2}\cmidrule(lr){3-5}
Caption featurization & Acc & Class & Quadrant & Stroke \\
\midrule
One-hot class (reference) & $0.724\pm0.033$ & $0.106$ & $1.000$ & $1.000$ \\
Per-slot indicators & $0.632\pm0.067$ & $0.114$ & $1.000$ & $0.843$ \\
Pooled text encoder & $0.464\pm0.049$ & $0.111$ & $0.886$ & $0.748$ \\
\quad without image-side projection & $0.32$ & --- & --- & --- \\
\quad without centering or intercept & $0.12$ & --- & --- & --- \\
\bottomrule
\end{tabular}
\end{table*}

\subsection{Experiments across featurizations}\label{sec:textmomexp}
How much of the Stage-1 gain survives when the finite class partition is replaced by the cross-moment release, as the conditioning signal moves from discrete labels toward free-form captions (Table~\ref{tab:textmom})?
Per-slot indicator features reach $0.632\pm0.067$ against the class-conditional reference $0.724\pm0.033$, whose release they generalize (the reference is the exact one-hot special case); the residual gap is attributable to inverting the noised $M_{cc}$.
Pooled frozen-encoder features reach $0.464\pm0.049$, and both geometric corrections of the Stage-1 release are essential: removing the image-side projection drops accuracy to $0.32$, and removing the centering and the released intercept to $0.12$, below the pipeline without any moment release.
On composed-MNIST the text-conditional release preserves multi-attribute controllability: indicator features match the class-conditional control on quadrant accuracy and approach it on stroke thickness, while class accuracy is at chance for all variants \emph{including the control}, so the shortfall on that attribute is a property of the task at this scale and budget, not of the release.
Both settings give the same ordering: the more discrete the caption featurization, the closer the release comes to its class-conditional ceiling, with the mechanism, sensitivity bound, and accounting unchanged throughout.
The residual gap is also budget-dependent in the direction the noise explanation predicts: for indicator features it narrows from $0.345\pm0.025$ vs.\ $0.443\pm0.047$ at $\epsdp{=}0.3$ through $0.632$ vs.\ $0.724$ at $\epsdp{=}1$ to $0.829\pm0.032$ vs.\ $0.848\pm0.004$ at $\epsdp{=}3$. As the budget grows, the noised $M_{cc}$ inversion sharpens and the text-conditional release converges to its class-conditional special case.

\section{Generation quality and empirical audits}\label{app:audits}

\subsection{Generation quality: marginals vs.\ conditionals}\label{sec:quality}
FID measures the \emph{marginal} sample distribution and is nearly blind to class mismatch: the baseline's FID ($62.7$) is close to ours ($66.7$) while its conditional utility is $3.5\times$ worse.
We therefore report FID/P/R alongside downstream accuracy rather than in place of it.
Figure~\ref{fig:grids} shows the samples behind these numbers: the baseline retains public-domain glyph structure at all budgets, the pixel-space models sharpen with budget, and the latent variant produces the most digit-like samples at every budget at the price of decoder smoothing (the precision artifact analyzed below).

\begin{figure}[t]
\centering
\includegraphics[width=0.92\columnwidth]{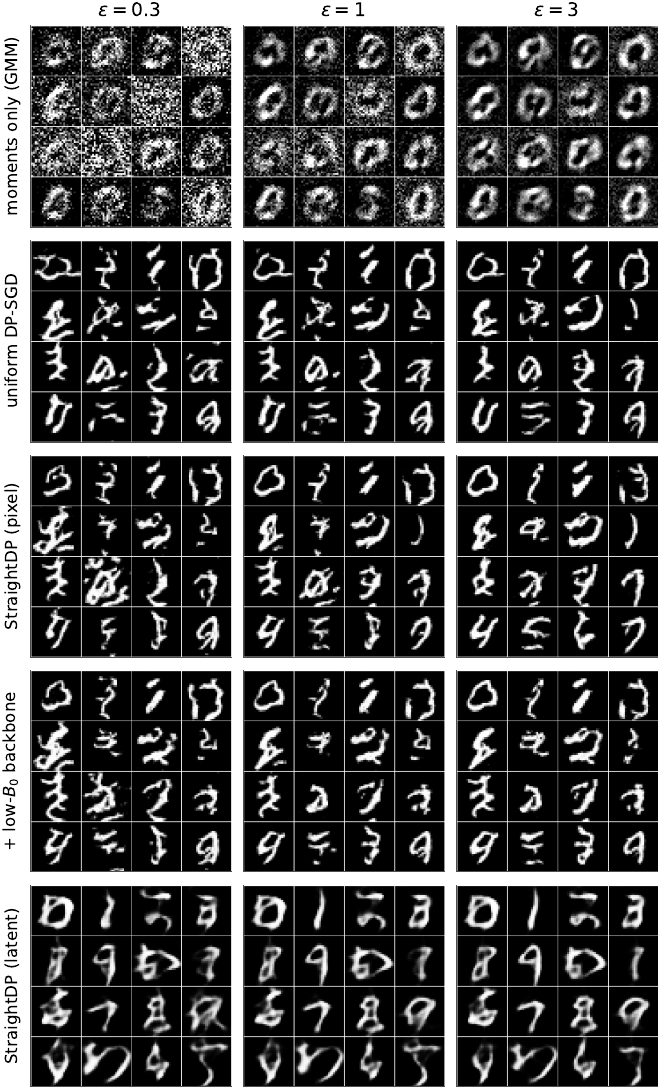}
\caption{Samples across budgets (rows: moments only, uniform DP-SGD, StraightDP pixel, $+$~low-$\Bo$ backbone, StraightDP latent; columns $\epsdp\in\{0.3,1,3\}$). Direct samples of the release are prototype-like; at strong privacy the DP-SGD baseline still emits public-domain glyphs; the release routes generation toward digits, and the latent variant is legible already at $\epsdp{=}0.3$.}
\label{fig:grids}
\end{figure}

Three label-aware metrics quantify the divergence directly at $\epsdp{=}1$ (Table~\ref{tab:condmetrics}, $10^4$ samples per run, 3 seeds): \emph{per-class FID}, the mean over classes of FID between generated and real test samples of the same class; \emph{label consistency}, the agreement of a classifier trained on real MNIST with the conditioning label of each generated sample; and \emph{domain rate}, the fraction of samples a real MNIST-vs-KMNIST classifier assigns to the private domain.
Conditioning FID on the class reverses the pixel-space ordering at this budget: uniform DP-SGD wins on label-blind FID ($62.7$ vs.\ $66.7$) but loses on per-class FID ($107.1$ vs.\ $101.4$), and its label consistency sits near the chance level of $0.1$, so its samples carry almost no information about the label they were conditioned on.
The released-moments sampler shows the mirror image, label consistency $0.62$ and the highest domain rate in the table at a per-class FID of $264$: the prototype signature of the main text in label-aware form.

\begin{table}[t]
\centering
\footnotesize
\setlength{\tabcolsep}{2pt}
\caption{Label-aware metrics at $\epsdp{=}1$ (MNIST, 3 seeds, mean$\pm$std; bold: best per column). pcFID = per-class FID; cons.\ = label consistency (chance $0.1$); domain = fraction assigned to the private domain by a real MNIST-vs-KMNIST classifier.}
\label{tab:condmetrics}
\begin{tabular}{lccc}
\toprule
& pcFID$\downarrow$ & Label cons. & Domain \\
\midrule
Moments only (GMM) & $264.4{\pm}2.9$ & $\mathbf{0.621{\pm}0.005}$ & $\mathbf{0.944{\pm}0.059}$ \\
Uniform DP-SGD & $107.1{\pm}0.2$ & $0.170{\pm}0.006$ & $0.649{\pm}0.072$ \\
StraightDP (pixel) & $\mathbf{101.4{\pm}2.4}$ & $0.444{\pm}0.049$ & $0.689{\pm}0.049$ \\
+ Low-$\Bo$ backbone & $112.6{\pm}0.3$ & $0.348{\pm}0.017$ & $0.795{\pm}0.069$ \\
\bottomrule
\end{tabular}
\end{table}

Table~\ref{tab:fidsweep} gives FID across the budget for all three training configurations (3 seeds).
Two patterns reinforce the caveat.
First, FID improves monotonically with the budget and the three methods converge by $\epsdp{=}3$: at weak privacy, quality is no longer a discriminator, even though downstream accuracy still separates the methods by $0.1$--$0.2$.
Second, at $\epsdp{=}0.3$ the pixel model has the \emph{worst} FID ($94.1$) yet a \emph{higher} downstream accuracy than the baseline ($0.443$ vs.\ $0.055$): the two metrics move in opposite directions exactly where the distinction matters most.
The point is sharpest against the prior methods, all run under our protocol (same $(\epsdp,\deltadp)$, replace-adjacency PLD accounting, public-only calibration, identical evaluation).
The one-shot embedding family attains strong downstream accuracy precisely because a labeled mean embedding is what a classifier probe needs, DP-MERF~\citep{harder2021_dpmerf} at $0.823$--$0.856$ across $\epsdp\in[0.3,10]$ and essentially flat in the budget, while its FID never approaches a sample-quality regime; neither the Hermite-feature successor DP-HP~\citep{vinaroz2022_dphp} nor the NTK-feature successor DP-NTK~\citep{yang2023_dpntk} improves on it once protocol, calibration, and generator are held fixed (their reported orderings come from per-dataset tuning).
Private Evolution~\citep{lin2024_pe}, run from the same frozen KMNIST prior as every pipeline row, fails in the opposite direction: its selection objective \emph{is} the Inception feature space, so with enough votes FID drops monotonically to $35.4\pm0.4$ at $\epsdp{=}10$, below every training route, while downstream accuracy never leaves the vicinity of chance ($0.044$--$0.111$); the evolved populations stay on the prior's Kuzushiji manifold, whose conditioning keeps regenerating glyph $c$ for caption $c$, so nearest-neighbor votes cannot rebind the class semantics. An aggressive variation schedule (renoising to $t_v\in[0.1,0.5]$ instead of $[0.45,0.85]$) lifts accuracy only to $0.134$ at $\epsdp{=}1$: the failure is structural, not an artifact of variation strength.
Each training-free mechanism saturates exactly the metric it optimizes and no other; we report both throughout.

\begin{table*}[t]
\centering
\small
\caption{FID vs.\ budget (MNIST, 3 seeds, mean$\pm$std; bold: best among the training configurations). The prior methods below the rule run under our protocol. The embedding family's probe accuracy is nearly flat in $\epsdp$ ($0.82$--$0.86$ for DP-MERF, $0.74$--$0.81$ for its successors) while its FID never approaches a usable regime; Private Evolution inverts the failure, optimizing Inception statistics directly, so its FID reaches $35.4$ while its accuracy stays at or below chance.}
\label{tab:fidsweep}
\begin{tabular}{lcccc}
\toprule
FID$\downarrow$ & $\epsdp{=}0.3$ & $\epsdp{=}1$ & $\epsdp{=}3$ & $\epsdp{=}10$ \\
\midrule
Uniform DP-SGD & $72.5\pm0.3$ & $62.7\pm0.1$ & $51.6\pm0.1$ & $41.4\pm0.2$ \\
StraightDP (pixel) & $94.1\pm2.7$ & $66.7\pm0.6$ & $51.5\pm0.3$ & $\mathbf{41.3\pm0.6}$ \\
StraightDP (latent) & $\mathbf{64.1\pm1.3}$ & $\mathbf{56.5\pm0.8}$ & $\mathbf{50.2\pm0.1}$ & $44.9\pm0.4$ \\
\midrule
DP-MERF & $280.6\pm0.8$ & $182.8\pm5.0$ & $109.1\pm2.4$ & $97.4\pm1.9$ \\
DP-HP & $284.3\pm5.0$ & $219.8\pm4.5$ & $167.3\pm5.6$ & $158.5\pm3.7$ \\
DP-NTK & $333.8\pm3.0$ & $292.7\pm3.4$ & $249.1\pm5.7$ & $229.8\pm11.1$ \\
Private Evolution & $122.4\pm2.1$ & $81.4\pm0.8$ & $52.8\pm2.0$ & $35.4\pm0.4$ \\
\bottomrule
\end{tabular}
\end{table*}

\subsection{Precision/recall: three structural signatures}\label{sec:pr}

Table~\ref{tab:prsweep} completes the marginal picture with improved precision/recall ($k{=}3$)~\citep{kynkaanniemi2019_pr}, computed in the same Inception feature space as FID: precision is the fraction of generated samples inside the real data's $k$-NN manifold (per-sample fidelity), recall the fraction of real samples inside the generated manifold (mode coverage).
Neither sees labels, and three patterns that look anomalous at first sight are in fact structural, each corroborated by a metric that moves the opposite way.

\emph{(i) The baseline wins the label-blind metrics while losing the labeled one.}
Uniform DP-SGD holds the best or tied-best P/R at every budget, with the worst downstream accuracy.
Its domain transfer fails: at strong privacy the weights barely leave the public pretraining, so it emits clean, diverse Kuzushiji-style glyphs (Figure~\ref{fig:grids}, main text) whose low-level stroke statistics sit close to the MNIST feature manifold.
A cleanly wrong marginal beats a noisily right one on any label-blind metric; the accuracy column inverts the ordering by up to $8\times$.
Our pixel model at $\epsdp{=}0.3$ incurs the mirror-image cost (P/R $0.15$/$0.38$, FID $94.1$): it genuinely moves toward the private domain and absorbs DP noise on the way.

\emph{(ii) The latent variant's precision collapses and is flat in the budget.}
Precision falls to $0.03$--$0.05$ while the same samples give the best accuracy at $\epsdp\le1$.
The frozen decoder smooths high-frequency structure and displaces \emph{every} sample by a similar offset in feature space, to just outside the tight $k$-NN balls of the nearly binary real digits, so the ``inside the ball'' count collapses wholesale.
The flatness in $\epsdp$ identifies the mechanism: the bottleneck is a decoder property, not a privacy-noise property.
The downstream classifier reads stroke topology rather than texture, so conditional information passes through the smoothing intact; recall survives at ${\approx}0.33$ because the generated manifold, with its own radii, still reaches a third of the real points.

\emph{(iii) The stream clamp contracts diversity.}
The low-$\Bo$ backbone has the lowest recall at strong privacy ($0.314$ at $\epsdp{=}0.3$): bounding activations trades expressive diversity for per-sample gradient signal-to-noise and concentrates generations on prototypical forms, the same trade that yields $+27\%$ accuracy (Table~\ref{tab:backbone}).

Absolute levels are depressed throughout (the best precision in the table is $0.26$): $28{\times}28$ grayscale digits upsampled to Inception's input resolution are a coarse instrument, a known artifact of feature-space P/R at MNIST scale, so the table should be read comparatively rather than absolutely.
Outside these three signatures, P and R rise monotonically with the budget, as expected.

\begin{table*}[t]
\centering
\small
\setlength{\tabcolsep}{2.8pt}
\caption{Improved precision/recall ($k{=}3$) vs.\ budget (MNIST, 3 seeds, mean$\pm$std); methods as in Table~\ref{tab:main}. The three structural signatures analyzed in the text: the label-blind metrics prefer the baseline's clean wrong-domain glyphs; the latent variant's precision collapse is decoder smoothing (flat in $\epsdp$); the low-$\Bo$ clamp contracts recall.}
\label{tab:prsweep}
\begin{tabular}{lcccccc}
\toprule
& \multicolumn{2}{c}{$\epsdp{=}0.3$} & \multicolumn{2}{c}{$\epsdp{=}1$} & \multicolumn{2}{c}{$\epsdp{=}3$} \\
\cmidrule(lr){2-3}\cmidrule(lr){4-5}\cmidrule(lr){6-7}
Method & P & R & P & R & P & R \\
\midrule
Uniform DP-SGD (baseline) & $0.164\pm0.002$ & $0.443\pm0.016$ & $0.219\pm0.002$ & $0.510\pm0.015$ & $0.246\pm0.001$ & $0.543\pm0.004$ \\
THA-planned DP-SGD & $0.167\pm0.013$ & $0.443\pm0.022$ & $0.219\pm0.002$ & $0.506\pm0.012$ & $0.248\pm0.008$ & $0.548\pm0.011$ \\
THA + moments (pixel) & $0.149\pm0.017$ & $0.381\pm0.038$ & $0.207\pm0.013$ & $0.471\pm0.012$ & $0.258\pm0.003$ & $0.519\pm0.013$ \\
\quad + Low-$\Bo$ backbone (pixel) & $0.138\pm0.014$ & $0.314\pm0.027$ & $0.228\pm0.013$ & $0.405\pm0.020$ & $0.233\pm0.005$ & $0.459\pm0.013$ \\
THA + moments (latent) & $0.032\pm0.001$ & $0.331\pm0.013$ & $0.039\pm0.003$ & $0.337\pm0.007$ & $0.046\pm0.002$ & $0.320\pm0.009$ \\
\bottomrule
\end{tabular}
\end{table*}

\subsection{Gradient concentration under the stream clamp}\label{app:gradstats}
Every run logs per-step quantiles of the per-sample gradient norms and the clip fraction; these logs quantify the mechanism claimed for the low-$\Bo$ gain.
At $\epsdp{=}0.3$ (pixel, 3 seeds, medians over steps): the unconstrained backbone has per-sample norm quantiles $p_{50}{=}1337$, $p_{90}{=}2156$ (upper-tail ratio $p_{90}/p_{50}=1.61$) at clip fraction $0.14$; the clamped backbone has $p_{50}{=}1450$, $p_{90}{=}2116$ (ratio $1.46$) at clip fraction $0.13$.
The clamp narrows the upper tail rather than shrinking the median: fewer samples are far above the clip, which is exactly the regime where clipping bias interacts worst with large $\sigma$.
These are training-trajectory statistics, not worst-case bounds; the claim in the main text is empirical.

\subsection{Attack audits}\label{sec:audits}
Accounted budgets are worst-case statements; we complement them with empirical attacks against the recommended configurations (Table~\ref{tab:audits}).
A membership-inference attack thresholds the flow-matching residual $\|v_\theta(tz+(1-t)\xi,t)-(z-\xi)\|$ of train versus held-out records, the statistic our training most directly optimizes; its AUC is indistinguishable from chance.
Canary extraction plants out-of-distribution secrets in the training set and ranks them by model likelihood after training; none of the $50$ canaries is preferentially generated.
The one-run audit inserts $200$ canaries in a single training run and derives a lower bound on $\epsdp$ from guessing accuracy; at accuracy $0.49$ the bound is vacuous, again consistent with the ledger.
Finally, every reported run's accounting transcript is replayed offline, and the recomputed totals never exceed the declared budgets.
Across all attacks the measured leakage lies far below the accounted $\epsdp$, as expected when the analyzed worst case is not attained by natural data.
\begin{table*}[t]
\centering
\footnotesize
\setlength{\tabcolsep}{4.5pt}
\caption{Attack audits on recommended configurations; all empirical leakage is consistent with (far below) the accounted budgets. Canary extraction follows \citet{carlini2019_secretsharer}; the one-run audit follows \citet{steinke2023_onerun} (200 canaries, 100 guesses).}
\label{tab:audits}
\begin{tabular}{lll}
\toprule
Audit & Result & Criterion \\
\midrule
Velocity-residual membership inference & AUC $0.482$ & $\approx0.5$ \\
Canary extraction ($50$ canaries) & $0/50$ extracted & indistinguishable from null \\
One-run auditing & guess accuracy $0.49$, $\epsdp_{\mathrm{lb}}=0$ & $\epsdp_{\mathrm{lb}}\le\epsdp$ \\
Transcript replay & achieved $\le$ declared, all runs & exact \\
\bottomrule
\end{tabular}
\end{table*}

\paragraph{Reproducibility statement}
All experiments run from a single repository with one command per run; every mechanism execution writes an auditable accounting transcript (JSONL) from which reported $\epsdp$ is recomputed; all calibrations use public data only; seeds and hyperparameters are enumerated in Section~\ref{app:hyper}. 

\paragraph{Ethics statement}
This work strengthens privacy protections for generative training and provides accounted guarantees plus empirical audits; we do not foresee direct negative applications beyond generic generative-model concerns.

\section{Hyperparameters}\label{app:hyper}

\paragraph{Compute infrastructure.}
Every run uses a single GPU on a Slurm cluster: NVIDIA L40S (48\,GB) or A100 (40\,GB) for training and evaluation, with the conditional metrics of Table~\ref{tab:condmetrics} computed on 32-core CPU nodes.
Software: Python 3.11, PyTorch 2.12 (CUDA 12.6), diffusers 0.39 and peft 0.19 (SD3 only), and the \texttt{dp-accounting} 0.6 PLD accountant; per-sample gradients use \texttt{torch.func} directly.
A pixel-space MNIST pipeline run completes in under two GPU-hours, a latent run in minutes, and one SD3 guidance evaluation in about twenty GPU-minutes.

All values in Table~\ref{tab:hyper} (final page) are read from the archived per-run configurations; none was tuned on private data (clip references, bucket shapes, $\tau$, and enforcement targets are public-data calibrations).

\textbf{Guidance start.}
The one guidance-specific hyperparameter is the start time $t_0$; Figure~\ref{fig:guidancet0} sweeps it and locates the operating point.

\textbf{End-to-end control.}
A two-dimensional Gaussian-mixture instance of the full pipeline runs in minutes and served as an end-to-end check during development; Figure~\ref{fig:toysweep} sweeps the budget on it.

\begin{figure}[t]
\centering
\includegraphics[width=\columnwidth]{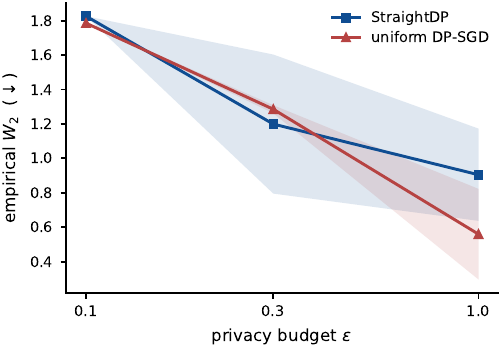}
\caption{Two-dimensional mixture control: Wasserstein-2 distance to the true distribution vs.\ budget (mean $\pm$ s.d.\ over up to 3 seeds; shaded bands). Both plans improve with budget and remain within seed noise of each other at this scale: the toy control verifies the pipeline end-to-end, but the problem is too small to separate the training plans; the separation emerges at MNIST scale (Table~\ref{tab:main}).}
\label{fig:toysweep}
\end{figure}

\begin{table*}[t]
\centering
\footnotesize
\caption{Hyperparameters. One column per pipeline component; shared values stated once. $\lambda_0$: covariance floor eigenvalue; $q$: Poisson sampling rate; $C_{\mathrm{ref}}$: reference clip before the public rescaling.}
\label{tab:hyper}
\setlength{\tabcolsep}{4pt}
\begin{tabular}{@{}ll@{}}
\toprule
\textbf{Backbone} & dual-stream DiT: width $192$, depth $6$, $6$ heads, patch $4$ ($28{\times}28$ input) \\
 & caption stream: frozen public encoder, width $64$, length $8$, vocabulary $256$ \\
 & constrained variants: $\Bo{=}16$, $\lo{=}1$ (decoupled attention), spectral cap $s{=}1$ \\
\midrule
\textbf{Public pretraining} & KMNIST (Fashion-MNIST in the transfer replication); $20{,}000$ steps, \\
 & batch $256$, AdamW lr $10^{-4}$, EMA $0.999$, $p_{\mathrm{uncond}}{=}0.1$ \\
 & latent variant: frozen public autoencoder; $12{,}000$ steps on latents \\
\midrule
\textbf{Stage 1 (moments)} & release weight $\rho_{\mathrm{mom}}{=}0.2$ (effective share $0.235\epsdp$ after \\
 & \quad renormalization over active stages); covariance rank $r{=}32$, $\lambda_0{=}10^{-2}$ \\
 & class release: public PCA offsets, feature clip $R_{\mathrm{feat}}$ (public quantile) \\
 & text release: image-side projection $p{=}64$, ridge $\lambda{=}10^{-3}$, \\
 & \quad $R_{\mathrm{txt}}$ at $1.1\times$ the public embedding-norm quantile \\
 & distillation: $3{,}000$ steps, batch $256$, lr $10^{-4}$, public replay; \\
 & \quad $\tau{=}0.35$, capped by the public cosine-knee calibration \\
 & \quad (the knee is public-domain stable: KMNIST and Fashion-MNIST \\
 & \quad publics both select $\tau{=}0.25$ in the latent variant) \\
\midrule
\textbf{Stage 2 (THA)} & DP-SGD weight $\rho_{\mathrm{sgd}}{=}0.65$ (effective share $0.765\epsdp$); $2{,}000$ steps, \\
 & \quad rate $q{=}0.02$ ($\E[\text{batch}]{=}1200$), \\
 & noise multiplicity $4$, lr $5{\times}10^{-5}$, EMA $0.999$, $p_{\mathrm{uncond}}{=}0.1$ \\
 & $K{=}4$ buckets, shapes $\beta\in\{0,\tfrac14,\tfrac12,\tfrac34,1\}$, $C_{\mathrm{ref}}{=}1$ publicly rescaled \\
 & accounting: exact PLD, $\deltadp{=}10^{-5}$; RDP bisection inside the planner \\
\midrule
\textbf{Prior methods} & DP-MERF: $10{,}000$ RFF, public median-heuristic bandwidth; generator \\
 & \quad $2{\times}512$ MLP, $6{,}000$ matching steps, Adam lr $10^{-3}$ (shared by all three) \\
 & DP-HP: sum kernel order $20$ $+$ product kernel over $1{,}500$ public-variance \\
 & \quad pixel pairs (order $3$, weight $\alpha{=}0.5$); DP-NTK: ReLU NTK, width $300$ \\
 & Private Evolution: $T{=}10$ vote rounds (PLD-composed), pop.\ $1{,}000$/class, \\
 & \quad $2$ variations, $t_v\in[0.45,0.85]$; SD3: $T{=}8$, $20$/class, $\sigma_v\in[0.35,0.8]$ \\
\midrule
\textbf{Evaluation} & $10{,}000$ samples, $25$ Euler steps, guidance $1.5$; LeNet probe for down- \\
 & stream accuracy (train on synthetic, test on real); FID and precision/ \\
 & recall ($k{=}3$) in Inception-V3 features; seeds $0$--$2$ throughout \\
\bottomrule
\end{tabular}
\end{table*}

\end{document}